\documentclass{article} 
\usepackage{arxiv_2027,times}

\usepackage{amsmath,amsfonts,bm}

\def\eqref#1{equation~\ref{#1}}

\def\1{\bm{1}}

\DeclareMathAlphabet{\mathsfit}{\encodingdefault}{\sfdefault}{m}{sl}
\SetMathAlphabet{\mathsfit}{bold}{\encodingdefault}{\sfdefault}{bx}{n}

\usepackage{graphicx}
\usepackage{amssymb}
\usepackage{amsmath}
\usepackage{booktabs}
\usepackage{float}
\usepackage{xcolor}
\usepackage{placeins}
\usepackage{flafter}
\usepackage{hyperref}
\usepackage{url}

\newcommand{\structtag}[1]{\texttt{\textless#1\textgreater}}

\newtheorem{proposition}{Proposition}[section]
\newenvironment{proof}{\par\noindent\textit{Proof.}\ }{\hfill$\square$\par}

\newcommand{\paperfigure}[4][]{%
  \IfFileExists{#2}{%
    \includegraphics[#1]{#2}%
  }{%
    \begingroup
    \setlength{\fboxsep}{0pt}%
    \fbox{\parbox[c][#4][c]{#3}{\centering\scriptsize
      Missing figure source}}%
    \endgroup
  }%
}

\title{CoEvolve: Construct-to-Edit Visual\\Grounding with Bidirectional\\State Refinement}

\author{
Dongwei Sun\textsuperscript{1}\thanks{Equal contribution.},
Yujie Zhang\textsuperscript{2}\footnotemark[1],
Bowen Yao\textsuperscript{2},
Pei Liu\textsuperscript{1},
Jing Yao\textsuperscript{3},
Xiangyong Cao\textsuperscript{1}\thanks{Corresponding author.}
\\[6pt]
\textsuperscript{1}
School of Computer Science and Technology and the Ministry of Education Key Lab for Intelligent \\ Networks and Network Security,
Xi'an Jiaotong University, Xi'an 710049, China
\\[3pt]
\textsuperscript{2}
School of Computer Science and Technology, Faculty of Electronic and Information Engineering,\\
Xi'an Jiaotong University, Xi'an 710049, China
\\[3pt]
\textsuperscript{3}
State Key Laboratory of Remote Sensing and Digital Earth, Aerospace Information Research Institute,\\
Chinese Academy of Sciences, Beijing 100094, China
\\[5pt]
\texttt{\{sundongwei,
zhang\_yujie2005,
BowenYao2007\}@outlook.com}
\\
\texttt{lpei53682@gmail.com,
jasonyao92@gmail.com,
caoxiangyong@xjtu.edu.cn}
}

\begin{document}

\maketitle

\begin{abstract}
Visual grounding localizes an object described by language with a bounding box. Most multimodal grounding models compress target identification, spatial reasoning, and boundary estimation into one terminal prediction. Free-form rationales make reasoning linguistically explicit but do not necessarily expose measurable, editable spatial states. Intermediate localization errors are therefore difficult to diagnose and correct, allowing incorrect region choices and imprecise boundaries to persist in the final box.
We introduce \textbf{CoEvolve}, a construct-to-edit framework that separates grounding into explicit state construction and state editing. Region-Evolution Reinforcement (RER) organizes grounding analysis into a progressive semantic--spatial trajectory, with each reasoning step committing to an explicit candidate region. Bidirectional Denoising Refiner (BDR) treats the reasoning text as fixed semantic context and refines the trajectory's coordinate fields through bidirectional same-position reconstruction. Geometry- and behavior-level objectives provide target geometry and edit-preference signals for consolidating reliable candidates, preserving accurate inputs, or correcting toward annotations.
Evaluations cover natural-image and remote-sensing grounding. With a 9B backbone, CoEvolve rivals models up to 241B parameters in grounding accuracy. Under controlled corruption, a single BDR pass improves mean box overlap by over 27 percentage points, demonstrating strong recovery from substantial localization errors. State-source comparisons further support the complementarity of explicit state construction and source-matched editing. The project is at \url{https://sundongwei.github.io/CoEvolve_Project/}.

\end{abstract}

\section{Introduction}

Visual grounding, or referring expression comprehension, localizes a language-described object with a bounding box. The field has progressed from proposal-based and modular grounding~\citep{kazemzadeh2014referitgame,mao2016generation,yu2016modeling,yu2018mattnet,zhang2018variational,liu2019erase,yang2019dynamic,liu2019nmtree} to end-to-end Transformers and grounded pre-training~\citep{deng2021transvg,kamath2021mdetr,li2022glip,yan2023uninext,liu2024groundingdino,wang2024onepeace}, and more recently to multimodal large language models (MLLMs) that unify language understanding and coordinate generation~\citep{you2024ferret,wang2024qwen2vl,wang2024cogvlm,bai2025qwen3vl,bai2025qwen25vl,wang2025internvl35,wu2024deepseekvl2}. Despite stronger semantic representations, most approaches still compress target disambiguation, spatial search, and boundary estimation into one terminal box. Without an explicit spatial state for inspection and repair, errors in region selection and boundary estimation can persist into the final prediction.
\begin{figure}[!ht]
    \centering
    \includegraphics[width=\linewidth]{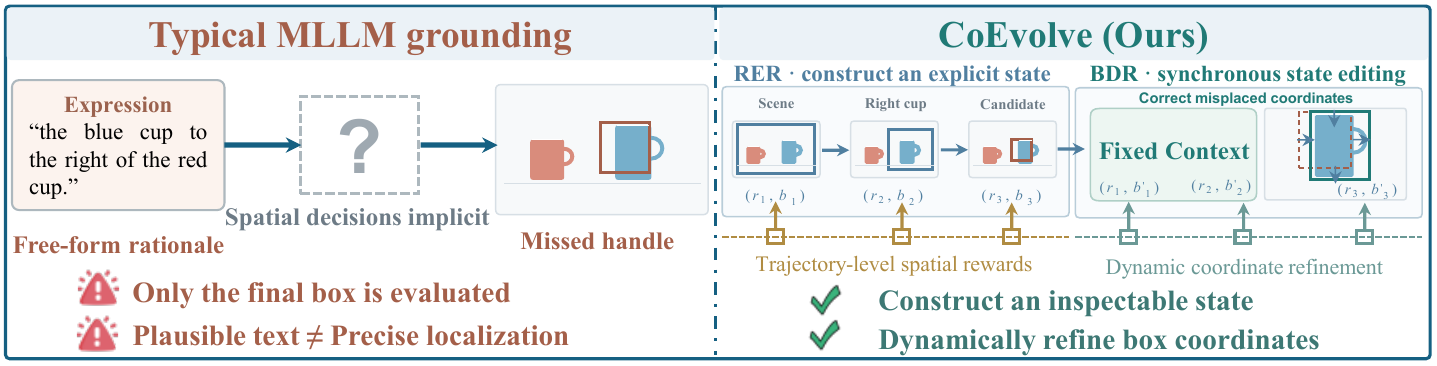}
    \caption{Comparison between conventional MLLM grounding and CoEvolve. Conventional methods compress spatial reasoning and localization into a terminal prediction, leaving intermediate spatial decisions implicit and difficult to correct. CoEvolve uses RER to construct a measurable progressive region trajectory and BDR to synchronously refine its coordinate fields while retaining the surrounding response structure.}
    \label{fig:compare}
\end{figure}
Reasoning-oriented systems expose rationales before spatial prediction~\citep{guo2025deepseekr1,huang2025visionr1,liu2025segzero}, and grounded reasoning methods further pair language with explicit visual regions or coordinates~\citep{fan2025grit,wu2026gcot}. Iterative localization has also used reinforcement-guided region search~\citep{sun2021iterative}, including anchor-guided reasoning in remote sensing~\citep{wang2026geosearcher}. These advances motivate a further question: how can candidate target regions form a measurable state that is passed to an editor, so that coordinate changes can be evaluated against the observed state and accurate regions can be preserved?

We introduce \textbf{CoEvolve}, a construct-to-edit framework that separates grounding into explicit state construction and state editing, as illustrated in Fig.~\ref{fig:compare}. \textbf{Region-Evolution Reinforcement (RER)} externalizes otherwise free-form grounding analysis into a progressive semantic--spatial trajectory in which every reasoning step commits to an explicit region. Trajectory-level spatial rewards derived from terminal annotations adapt the model to this structured interface without human-authored \mbox{reasoning} chains or intermediate region labels. \textbf{Bidirectional Denoising Refiner (BDR)} retains the reasoning text as semantic context and edits the trajectory's coordinate fields through bidirectional same-position reconstruction. Geometry-level supervision provides target geometry, while behavior-level supervision provides preferences for consolidating a reliable candidate, preserving an accurate input, or correcting toward the annotation.

Experiments evaluate localization accuracy and state refinement. CoEvolve reaches 92.41\% Overall $\mathrm{Acc}@0.5$ on eight RefCOCO-family splits and 85.72\% mIoU on five test splits. Under controlled corruption, BDR raises mIoU from 58.80\% to 86.19\%, with 87.99\% $\mathrm{Repair}@0.7$. Domain-specific retraining achieves 76.23\% mIoU on DIOR-RSVG. Our contributions are:

\begin{itemize}
    \item We introduce CoEvolve and its RER component to expose intermediate localization decisions as a parseable semantic--spatial trajectory, providing an explicit state for inspection and correction.
    \item We develop BDR to edit the coordinate fields of this state through position-aligned same-position reconstruction. Geometry- and behavior-level objectives jointly provide target geometry and edit-preference signals for state refinement.
    \item We demonstrate strong target localization across natural-image and remote-sensing benchmarks, together with substantial correction gains under controlled box corruption.
\end{itemize}

\section{Preliminaries}
\label{sec:preliminaries}

\subsection{Group Relative Policy Optimization}
\label{sec:prelim_grpo}

Let $c=(I,q)$ denote a multimodal context consisting of an image $I$ and a referring expression $q$, and let $y=(y_1,\ldots,y_T)$ be a response containing the reasoning trajectory and final coordinates. An autoregressive policy factorizes the response likelihood as
\begin{equation}
\pi_\theta(y\mid c)
=
\prod_{t=1}^{T}
\pi_\theta(y_t\mid c,y_{<t}).
\label{eq:ar_policy}
\end{equation}

Group Relative Policy Optimization (GRPO)~\citep{shao2024deepseekmath} removes the learned critic by sampling a group $\{y_i\}_{i=1}^{G}\sim\pi_{\theta_{\mathrm{old}}}(\cdot\mid c)$ and normalizing sequence-level rewards $R_i$ within the group:
\begin{equation}
A_i
=
\frac{R_i-\mu_R}{\sigma_R+\delta},
\qquad
\mu_R
=
\frac{1}{G}\sum_{i=1}^{G}R_i,
\label{eq:group_advantage}
\end{equation}
where $\sigma_R$ is the corresponding standard deviation and $\delta$ ensures numerical stability. With the token-wise importance ratio
\[
\rho_{i,t}(\theta)
=
\frac{
    \pi_\theta(y_{i,t}\mid c,y_{i,<t})
}{
    \pi_{\theta_{\mathrm{old}}}(y_{i,t}\mid c,y_{i,<t})
},
\]
the policy is optimized by
\begin{align}
\mathcal{J}_{\mathrm{GRPO}}(\theta)
=\mathbb{E}\!\Bigg[&\frac{1}{G}\sum_{i=1}^{G}\frac{1}{T_i}
\sum_{t=1}^{T_i}\min\!\Big\{\rho_{i,t}(\theta)A_i,\nonumber\\[-2pt]
&\operatorname{clip}\!\left(\rho_{i,t}(\theta),1-\epsilon_l,1+\epsilon_h\right)A_i\Big\}
-\beta D_{\mathrm{KL}}\!\left(\pi_\theta\Vert\pi_{\mathrm{ref}}\right)\Bigg].
\label{eq:grpo_objective}
\end{align}
Here $\pi_{\mathrm{ref}}$ is the reference policy, $\beta$ controls KL regularization, and $\epsilon_l$ and $\epsilon_h$ are the lower and upper clipping radii. Setting $\epsilon_l=\epsilon_h$ recovers symmetric clipping; using $\epsilon_h>\epsilon_l$ permits broader exploration for positively advantaged responses.

\subsection{Bidirectional Denoising Refinement}
\label{sec:prelim_denoising}

Bidirectional reconstruction over corrupted discrete tokens provides a natural bridge between bidirectional context modeling and iterative correction~\citep{ou2024absorbing,sahoo2024masked,sun2026echochange}. Lavida-O applies masked diffusion with bidirectional context to discretized grounding coordinates~\citep{li2026lavidao}. BDR instead keeps an existing response and its coordinate tokens visible as references for same-position editing. Let $z=(z_1,\ldots,z_L)$ denote a clean response and $\widetilde{z}$ its discretely corrupted version, with $\mathcal{M}$ denoting the perturbed positions. A bidirectional denoiser reconstructs the original token at each selected position with the same-position objective:
\begin{equation}
\mathcal{L}_{\mathrm{denoise}}
=
-\mathbb{E}_{z,\widetilde{z}}\!\left[
\frac{1}{|\mathcal{M}|}
\sum_{j\in\mathcal{M}}
\log p_\phi(z_j\mid\widetilde{z},c)_j
\right].
\label{eq:denoising_objective}
\end{equation}

Here $p_\phi(\cdot\mid\widetilde z,c)_j$ is the token distribution predicted at response position $j$. Starting from an imperfect response $z^{(0)}$, the same position-aligned predictor can update selected positions for $P$ passes:
\begin{equation}
z_j^{(p)}
=
\begin{cases}
\displaystyle
\operatorname*{arg\,max}_{v}\;
p_\phi\!\left(v\mid z^{(p-1)},c\right)_j,
& j\in\mathcal{M}^{(p)},\\
z_j^{(p-1)},
& \text{otherwise},
\end{cases}
\qquad
p=1,\ldots,P.
\label{eq:iterative_correction}
\end{equation}

Here $\mathcal{M}^{(p)}$ denotes the set of response positions updated during
refinement pass $p$.

This formulation combines bidirectional position-aligned reconstruction with denoising-style iterative correction. It operates directly on discrete response tokens rather than introducing a continuous noise-time process or a full reverse-time Markov chain. In CoEvolve, $z^{(0)}$ is the RER localization trajectory, and BDR extends this view to visible and geometrically corrupted coordinate tokens; token masking is used only during warm-up. Bidirectional context and coordinate-aware objectives define the refinement process described in Section~\ref{sec:method}.

\section{Methodology}
\label{sec:method}

\subsection{Overview}
\label{sec:overview}

Given an image $I$ and a referring expression $q$, visual grounding predicts the referred object's box in the non-empty target setting. We denote the training annotation by $\mathbf b^{*}=[x_1^{*},y_1^{*},x_2^{*},y_2^{*}]$ and normalize coordinates to integers in $[0,1000]$. A valid box satisfies $x_1<x_2$ and $y_1<y_2$ within this range. Fig.~\ref{fig:overview} illustrates the workflow. \textbf{Region-Evolution Reinforcement (RER)} constructs a tagged response $y$ comprising reasoning text, a progressive region sequence, and a terminal answer. \textbf{Bidirectional Denoising Refiner (BDR)} retains the reasoning text as semantic context and refines the state's coordinate fields. This state provides semantic context, spatial history, an edit reference, and natural upstream errors. The two components are trained sequentially and communicate only through the generated state.

\begin{figure*}[ht]
    \centering
    \includegraphics[width=\linewidth]{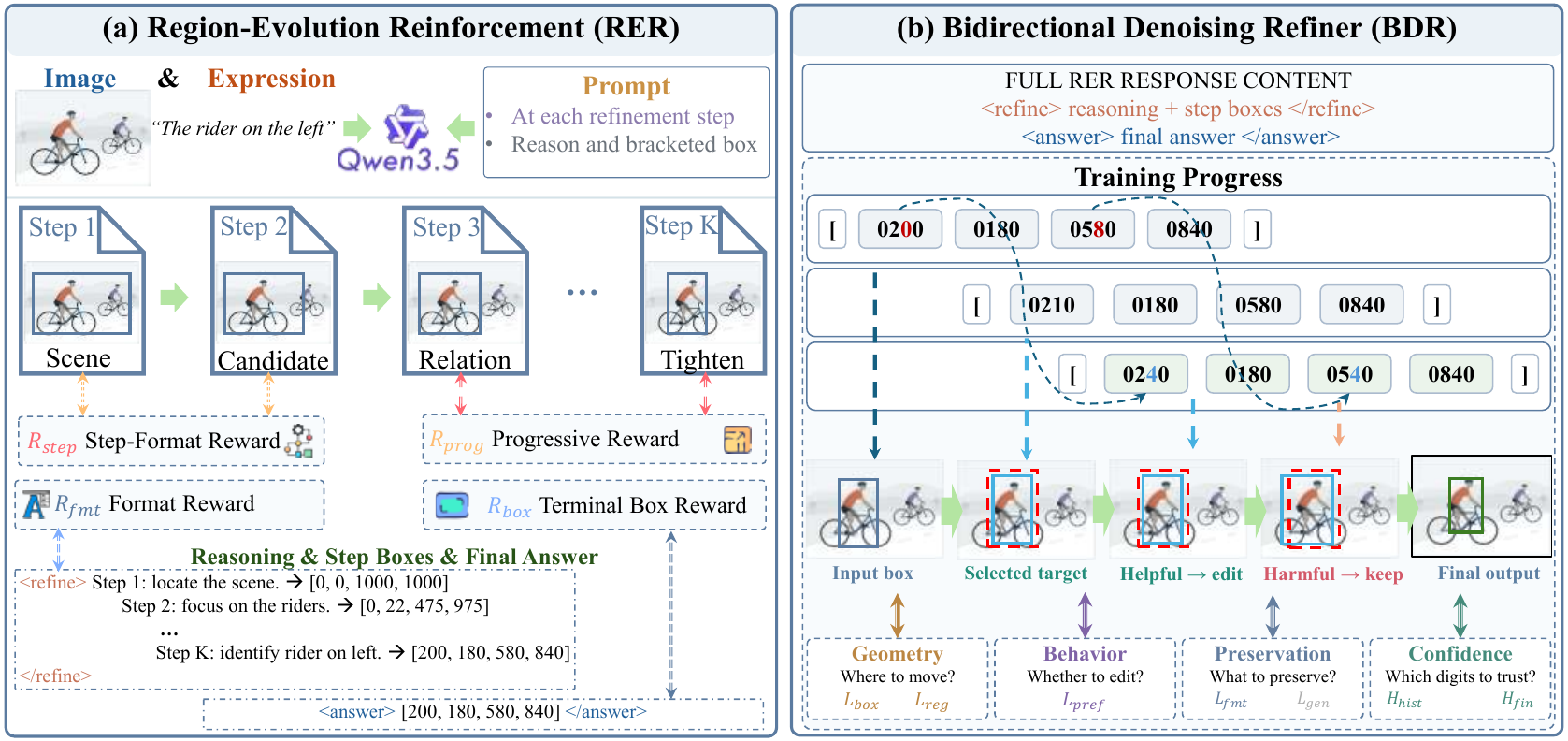}
    \caption{Overview of CoEvolve. RER converts otherwise free-form grounding analysis into a progressive semantic--spatial state in which each reasoning step commits to an explicit region. The resulting response becomes an editable localization state for BDR. BDR performs bidirectional same-position refinement of its coordinate fields while retaining the surrounding response as context. Geometry-level objectives provide target geometry, and behavior-level objectives provide edit-preference signals.}
    \label{fig:overview}
\end{figure*}

\subsection{Region-Evolution Reinforcement}
\label{sec:rer}

An outcome-only box reward does not directly assess whether intermediate region states make target-directed progress before the terminal prediction. RER therefore adapts an autoregressive policy $\pi_{\theta}(y\mid I,q)$ to construct a measurable localization state. A valid response contains one \structtag{refine} block with $3$--$5$ steps and one \structtag{answer} block. We denote its parsed sequence by $\tau(y)$, where each pair $(r_k,\mathbf b_k)$ associates a reasoning step with an explicit region and $\mathbf b_{\mathrm{ans}}$ is the terminal answer. Unlike regions that identify visual evidence in grounded reasoning~\citep{fan2025grit,wu2026gcot}, each RER box is a candidate location for the referred target, and the ordered sequence becomes BDR's editable state. We use \emph{semantic--spatial state} to denote this division of roles: free-form text supplies semantic context, while each step exposes a spatial commitment evaluated by trajectory-level rewards. These rewards assess the spatial commitments rather than the semantic correctness of the free-form text itself.
For each sampled response, RER assigns the sequence-level reward
\begin{equation}
\begin{aligned}
R_{\mathrm{RER}}(y,\mathbf b^{*})={}&
0.5R_{\mathrm{fmt}}(y)+0.5R_{\mathrm{step}}(y)+R_{\mathrm{prog}}(y,\mathbf b^{*})
+R_{\mathrm{box}}(y,\mathbf b^{*}).
\end{aligned}
\label{eq:rer_reward}
\end{equation}
$R_{\mathrm{fmt}}$ checks the outer tagged structure and provides graded feedback for partially valid responses. $R_{\mathrm{step}}$ checks whether each refinement step ends in a parseable box; it does not score localization accuracy. When $3$--$5$ valid non-empty regions are available, $R_{\mathrm{prog}}$ evaluates target-directed spatial progress using overlap, center, coverage, and final-consistency signals. $R_{\mathrm{box}}$ parses and evaluates the terminal answer by its overlap and shape agreement with $\mathbf b^{*}$, and assigns the prescribed negative reward when the answer is missing, repeated, malformed, or invalid. The appendix specifies the piecewise definitions and coefficients of these terms.

We sample eight responses per input and use Eq.~(\ref{eq:rer_reward}) in the GRPO objective introduced in Section~\ref{sec:preliminaries}. The annotated terminal box evaluates the generated regions, but RER requires no manually labeled reasoning text or intermediate box trajectory. It therefore derives trajectory-level spatial rewards from the terminal box annotation and organizes the search process into a structured state whose coordinate fields BDR can refine.

\subsection{Bidirectional Denoising Refiner}
\label{sec:bdr}

\paragraph{Editable localization states.}
BDR refines complete RER responses, using their fixed reasoning text as semantic context and editing the coordinate fields of historical Steps, the last Step, and the Answer. Geometry-level and behavior-level objectives supervise target coordinates and edit decisions, while auxiliary reconstruction and entropy terms preserve the response and regularize confidence. For states whose RER Answer has IoU above $0.5$, we sample $\xi\sim\operatorname{Bernoulli}(0.3)$; for all other states, we set $\xi=0$. Given $\mathbf x^{(0)}=\operatorname{Tok}(y)$, we use $\mathbf x^{\mathrm{in}}=\mathbf x^{(0)}$ when $\xi=0$ and $\mathbf x^{\mathrm{in}}=\operatorname{Tok}(\mathcal Q(y,\mathbf b^{*}))$ otherwise, where $\mathcal Q$ replaces selected boxes with valid geometric perturbations.

\paragraph{Bidirectional same-position reconstruction.}
We realize bidirectional state access by removing causal masking from the backbone's full-attention layers, while retaining each editable coordinate at its original response position. The set $\mathcal T_s$ contains the $16$ ordered digit positions of box role $s$, and $\mathcal T_{\mathrm{bbox}}=\bigcup_s\mathcal T_s$. BDR produces position-aligned logits for all editable coordinates from the same observed response. Denoting the ten digit tokens by $\mathcal V_{\mathrm{digit}}$, the coordinate candidates used during training are obtained as
\begin{equation}
\begin{aligned}
\mathbf z&=F_{\phi}(I,q,\mathbf x^{\mathrm{in}}),
\quad
\tilde{x}_t^{\mathrm{out}}
=\arg\max_{v\in\mathcal V_{\mathrm{digit}}}z_{t,v},
\quad t\in\mathcal T_{\mathrm{bbox}},\\[-2pt]
\tilde{\mathbf b}_s^{\mathrm{out}}
&=\operatorname{DecBox}\!\left(
(\tilde{x}_t^{\mathrm{out}})_{t\in\mathcal T_s}
\right).
\end{aligned}
\label{eq:bdr_reconstruction}
\end{equation}
Here $F_{\phi}$ denotes BDR with adapter parameters $\phi$, and $\mathbf z$ contains position-aligned logits. Digit projection yields synchronous box candidates for target selection, regression, and preference supervision. At inference, BDR updates coordinate positions while retaining the reasoning text; denoising denotes reconstruction from masking and geometric corruption, not a diffusion timestep process.

\paragraph{Behavior-Level Supervision: Quality-Gated Target Selection.}
We compute $u_{\mathrm{in}}=\operatorname{IoU}(\mathbf b_{\mathrm{ans}}^{\mathrm{in}},\mathbf b^{*})$ and $u_{\mathrm{out}}=\operatorname{IoU}(\tilde{\mathbf b}_{\mathrm{ans}}^{\mathrm{out}},\mathbf b^{*})$, assigning zero to an invalid box. With $\eta=0.9$ and $\gamma=0.005$, $\mathcal C_{\mathrm{out}}$ denotes $u_{\mathrm{out}}>\eta$ and $u_{\mathrm{out}}-u_{\mathrm{in}}\geq\gamma$, while $\mathcal C_{\mathrm{in}}$ denotes $\xi=0$ and $u_{\mathrm{in}}\geq\eta$. Absolute box supervision uses
\begin{equation}
\mathbf b^{\mathrm{tar}}=
\begin{cases}
\tilde{\mathbf b}_{\mathrm{ans}}^{\mathrm{out}}, & \mathcal C_{\mathrm{out}},\\
\mathbf b_{\mathrm{ans}}^{\mathrm{in}},          & \mathcal C_{\mathrm{in}},\\
\mathbf b^{*},                                      & \text{otherwise}.
\end{cases}
\label{eq:bdr_target}
\end{equation}
The stop-gradient branches define three target-selection cases. When $\mathcal C_{\mathrm{out}}$ holds, the candidate itself becomes the target, reinforcing a reliable improvement. When $\mathcal C_{\mathrm{in}}$ holds, the input anchor becomes the target, reinforcing preservation. Otherwise, the annotation provides the corrective target. The rule therefore supplies both the geometric label and the corresponding training preference for the current state.

\paragraph{Geometry-Level Supervision.}
Geometry-level supervision provides the target geometry for the edited localization. The last refinement box and the answer, denoted by $\mathcal S_f=\{\mathrm{last},\mathrm{ans}\}$, receive the same selected target from Eq.~(\ref{eq:bdr_target}). Their four zero-padded coordinates contribute $32$ strongly supervised digit positions. The coordinate cross-entropy $\mathcal L_{\mathrm{box}}$ enforces the target digit at each position, with equal digit weights during warm-up. After warm-up, the four digits of coordinate $k$ share $w_k=1+10e_k$, where $e_k=(2000)^{-1}\sum_{s\in\mathcal S_f}|\tilde b_{sk}^{\mathrm{out}}-b_k^{\mathrm{tar}}|$. The weighted token losses are divided by the sum of their weights, so $w_k$ redistributes gradient toward inaccurate coordinates without changing the nominal scale of $\mathcal L_{\mathrm{box}}$.

Token cross-entropy distinguishes exact digit identities but does not directly encode numerical distance. We therefore project the logits of digit $j\in\{0,1,2,3\}$ in coordinate $k\in\{1,2,3,4\}$ of role $s\in\mathcal S_f$ onto the ten digit tokens. For its position $t_{skj}$, the resulting distribution and soft coordinate are
\begin{equation}
\begin{aligned}
p_{skj}(d)&=
\frac{\exp z_{t_{skj},d}}
{\sum_{a=0}^{9}\exp z_{t_{skj},a}},
\bar b_{sk}=\sum_{j=0}^{3}10^{3-j}\sum_{d=0}^{9}d\,p_{skj}(d).
\end{aligned}
\label{eq:soft_coordinate}
\end{equation}
The regression loss $\mathcal L_{\mathrm{reg}}$ averages $|\bar b_{sk}-b_k^{\mathrm{tar}}|/1000$ over the eight supervised coordinates. Together, $\mathcal L_{\mathrm{box}}$ and $\mathcal L_{\mathrm{reg}}$ provide complementary geometric signals: the former preserves the exact discrete coordinate representation, while the latter encodes continuous numerical distance to the selected target.

\paragraph{Coordinate-wise Edit Preference.}
Quality-gated target selection defines the desired state-level behavior, but does not explicitly compare an individual coordinate edit with retaining its input value. We therefore introduce coordinate-wise preference supervision. We index historical roles by $\mathcal S_h=\{1,\ldots,K-1\}$ and define $\mathcal S=\mathcal S_h\cup\mathcal S_f$. For every role with $16$ valid input digits, $\breve{\mathbf b}_s^{\mathrm{in}}$ is decoded from the observed input and $\tilde{\mathbf b}_s^{\mathrm{out}}$ is the digit-projected candidate. Their coordinate-level edit gain is
\begin{equation}
\Delta_{sk}=
\frac{|\breve b^{\mathrm{in}}_{sk}-b_k^{*}|
-|\tilde b^{\mathrm{out}}_{sk}-b_k^{*}|}{1000}.
\label{eq:edit_gain}
\end{equation}
For coordinate $k$ of role $s$, $\mathcal C_{sk}$ contains the changed digit positions, $D_{sk}$ is their mean output-versus-input logit margin, and $y_{sk}$ indicates whether $\Delta_{sk}$ is positive. Thus, $\Delta_{sk}>0$ identifies an edit closer to the annotation than its input counterpart, whereas $\Delta_{sk}<0$ identifies a detrimental edit. Binary cross-entropy with logits is applied when $\mathcal C_{sk}\neq\varnothing$ and $|\Delta_{sk}|>0.005$. Active coordinates within a role are averaged with weights $a_{sk}=\min(1,|\Delta_{sk}|/0.05)$. Historical roles receive weight $0.2$, and the last Step and Answer each receive weight $0.5$. The active role losses are summed per sample and then averaged over the minibatch to obtain $\mathcal L_{\mathrm{pref}}$. Unlike absolute coordinate supervision, $\mathcal L_{\mathrm{pref}}$ learns the relative utility of an edit.

\paragraph{State Preservation and Confidence Regularization.}
The primary objectives operate inside a structured response whose reasoning and syntax must remain stable. General reconstruction $\mathcal L_{\mathrm{gen}}$ applies token cross-entropy outside the strongly supervised final digits; after warm-up, it preserves non-coordinate assistant content. Format loss $\mathcal L_{\mathrm{fmt}}$ adds cross-entropy on structural tags and coordinate punctuation. For a digit distribution $p$, natural entropy is $H(p)=-\sum_{d=0}^{9}p(d)\log p(d)$. The historical term $\mathcal H_{\mathrm{hist}}$ sharpens digit distributions in earlier refinement steps without imposing an absolute coordinate target. The final term $\mathcal H_{\mathrm{fin}}$ acts on the last Step and Answer only when the digit-vocabulary argmax equals the selected target digit, sharpening correct predictions without increasing confidence in an incorrect digit. These terms preserve response structure and non-coordinate content while regularizing confidence; they do not determine whether an edit is beneficial.

\paragraph{Combined objective and warm-up.}
With the coefficients used in training, the post-warm-up objective is
\begin{equation}
\begin{aligned}
\mathcal L_{\mathrm{BDR}}={}&
\underbrace{\mathcal L_{\mathrm{box}}+30\mathcal L_{\mathrm{reg}}}_{\mathcal L_{\mathrm{geom}}}
+\mathcal L_{\mathrm{pref}}\\
&+\underbrace{0.5\mathcal L_{\mathrm{fmt}}+0.2\mathcal L_{\mathrm{gen}}}_{\mathcal L_{\mathrm{state}}}
+\underbrace{0.2\mathcal H_{\mathrm{hist}}+0.2\mathcal H_{\mathrm{fin}}}_{\mathcal L_{\mathrm{conf}}}.
\end{aligned}
\label{eq:bdr_objective}
\end{equation}
Behavior-level supervision enters through both the target policy in Eq.~(\ref{eq:bdr_target}), which determines the target used by $\mathcal L_{\mathrm{geom}}$, and the additive preference term $\mathcal L_{\mathrm{pref}}$. During the first $3\%$ of optimization, we mask $10\%$ of eligible assistant tokens and use $\mathcal L_{\mathrm{box}}+0.2\mathcal L_{\mathrm{gen}}$. The target rule remains active; $\mathcal L_{\mathrm{gen}}$ also covers historical box digits, assigning the selected target to perturbed historical boxes. Regression, entropy, and preference are activated after warm-up.

\section{Experiments}
\label{sec:experiments}

We evaluate CoEvolve on standard referring-expression grounding, localization under stricter IoU criteria, controlled box correction, component ablation, and remote-sensing grounding. Unless otherwise noted, aggregate evaluation metrics are reported as percentages.

\subsection{Experimental Setup}
\label{sec:exp_setup}

\paragraph{Datasets and Evaluation.}
For natural-image grounding, we train CoEvolve only on the RefCOCOg training set and evaluate all eight RefCOCO-family splits (57,457 examples) without further fine-tuning~\citep{mao2016generation,yu2016modeling}. RefCOCOg measures in-dataset performance, while RefCOCO and RefCOCO+ provide cross-dataset evaluation under RefCOCOg-only training. We report split-wise $\mathrm{Acc}@0.5$ and their arithmetic mean (Overall). Stricter-IoU, controlled-correction, and ablation experiments use five test splits comprising 30,969 examples. Unless stated otherwise, their aggregate metrics are computed jointly over all examples rather than by averaging the five split-level scores. Remote-sensing grounding is evaluated on DIOR-RSVG~\citep{zhan2023rsvg}.

\subsection{Standard RefCOCO-Family Grounding}
\label{sec:main_grounding_v3}

Among the general-purpose MLLMs, specialized grounding systems, and grounding-adapted MLLMs compared in Table~\ref{tab:main_grounding_v3}, CoEvolve obtains the highest tabulated Overall $\mathrm{Acc}@0.5$ of 92.41, outperforming Qwen3.5-9B on all eight splits and improving its aggregate score by 2.93 points. CoEvolve ranks first on four splits and ties for first on RefCOCOg val, demonstrating strong and consistent grounding performance across the RefCOCO family.

\begin{table}[ht]
\centering
\setlength{\belowcaptionskip}{5pt}
\caption{Grounding performance on the eight standard RefCOCO-family splits, measured by $\mathrm{Acc}@0.5$. Overall is the arithmetic mean across the eight split-level results. It is computed from full-precision values when available and otherwise from the reported split-level values. Public methods follow their reported protocols; $^{\ddagger}$ denotes our locally reproduced Qwen3.5-9B baseline. CoEvolve trains only on RefCOCOg train, without target-dataset fine-tuning. Best and second-best results are bolded and underlined.}
\label{tab:main_grounding_v3}
{\scriptsize
\setlength{\tabcolsep}{2.0pt}
\renewcommand{\arraystretch}{1.0}
\begin{tabular*}{\textwidth}{@{\extracolsep{\fill}}lccccccccc@{}}
\toprule
& \multicolumn{3}{c}{RefCOCO} & \multicolumn{3}{c}{RefCOCO+} & \multicolumn{2}{c}{RefCOCOg} & \\
\cmidrule(lr){2-4}\cmidrule(lr){5-7}\cmidrule(lr){8-9}
Method & val & testA & testB & val & testA & testB & val & test & Overall $\uparrow$ \\
\midrule
Qwen2-VL-7B~\citep{wang2024qwen2vl} & 91.7 & 93.6 & 87.3 & 85.8 & 90.5 & 79.5 & 87.3 & 87.8 & 87.94 \\
Qwen2-VL-72B~\citep{wang2024qwen2vl} & 93.2 & 95.3 & 90.7 & 90.1 & \underline{93.8} & 85.6 & 89.9 & 90.4 & 91.13 \\
Qwen2.5-VL-7B~\citep{bai2025qwen25vl} & 90.0 & 92.5 & 85.4 & 84.2 & 89.1 & 76.9 & 87.2 & 87.2 & 86.56 \\
Qwen2.5-VL-72B~\citep{bai2025qwen25vl} & 92.7 & 94.6 & 89.7 & 88.9 & 92.2 & 83.7 & 89.9 & 90.3 & 90.25 \\
Qwen3-VL-8B-Instruct~\citep{bai2025qwen3vl} & 91.6 & 93.3 & 87.8 & 85.8 & 90.3 & 79.9 & 88.7 & 88.7 & 88.26 \\
Qwen3-VL-235B-A22B~\citep{bai2025qwen3vl} & 92.9 & 94.6 & 90.2 & 88.6 & 92.4 & 84.2 & 90.3 & 90.8 & 90.50 \\
InternVL3-8B~\citep{zhu2025internvl3} & 92.5 & 94.6 & 88.0 & 88.2 & 92.5 & 81.8 & 89.6 & 90.0 & 89.65 \\
InternVL3-78B~\citep{zhu2025internvl3} & 93.4 & \underline{95.4} & 90.3 & 90.1 & \underline{93.8} & 85.3 & 91.5 & 91.5 & 91.41 \\
InternVL3.5-8B~\citep{wang2025internvl35} & 92.4 & 94.7 & 88.7 & 87.9 & 92.4 & 82.4 & 89.6 & 89.4 & 89.69 \\
InternVL3.5-241B-A28B~\citep{wang2025internvl35} & \underline{94.1} & \textbf{96.3} & \underline{91.5} & \textbf{91.6} & \textbf{94.6} & \underline{86.9} & \underline{92.0} & 92.1 & \underline{92.39} \\
DeepSeek-VL2-Small~\citep{wu2024deepseekvl2} & 93.9 & 95.3 & 91.3 & 89.4 & 92.9 & 84.8 & \textbf{92.6} & \underline{92.6} & 91.60 \\
Grounding-DINO-L~\citep{liu2024groundingdino} & 90.6 & 93.2 & 88.2 & 82.8 & 89.0 & 75.9 & 86.1 & 87.0 & 86.60 \\
UNINEXT-H~\citep{yan2023uninext} & 92.6 & 94.3 & \underline{91.5} & 85.2 & 89.6 & 79.8 & 88.7 & 89.4 & 88.89 \\
ONE-PEACE~\citep{wang2024onepeace} & 92.6 & 94.2 & 89.3 & 88.8 & 92.2 & 83.2 & 89.2 & 89.3 & 89.85 \\
Ferret-v2-13B~\citep{zhang2024ferretv2} & 92.6 & 95.0 & 88.9 & 87.4 & 92.1 & 81.4 & 89.4 & 90.0 & 89.60 \\
CogVLM-Grounding-17B~\citep{wang2024cogvlm} & 92.8 & 94.8 & 89.0 & 88.7 & 92.9 & 83.4 & 89.8 & 90.8 & 90.28 \\
EGM-InternVL3-8B~\citep{zhan2026egm} & 93.6 & 95.2 & 90.1 & 89.3 & 93.6 & 83.1 & 89.7 & 90.7 & 90.66 \\
EGM-Qwen3-VL-8B~\citep{zhan2026egm} & 93.9 & 95.0 & 91.2 & 90.1 & 93.3 & 85.9 & 90.4 & 91.2 & 91.38 \\
Qwen3.5-9B$^{\ddagger}$~\citep{qwen2026qwen35} & 91.5 & 93.8 & 89.5 & 87.7 & 91.3 & 82.9 & 89.7 & 89.6 & 89.48 \\
\midrule
\textbf{CoEvolve (ours)} & \textbf{94.4} & 95.3 & \textbf{92.4} & \underline{90.7} & 93.5 & \textbf{87.8} & \textbf{92.6} & \textbf{92.8} & \textbf{92.41} \\
\bottomrule
\end{tabular*}}
\end{table}

\subsection{Localization under Stricter IoU Criteria}
\label{sec:stringent_iou_v3}

Table~\ref{tab:stringent_iou_v4} evaluates localization under stringent IoU criteria across five test splits. CoEvolve achieves 85.72\% mIoU, 87.94\% $\mathrm{Acc}@0.7$, and 68.64\% $\mathrm{Acc}@0.9$, exceeding CogVLM-Grounding-17B by 2.12, 2.63, and 4.19 points, respectively. CoEvolve ranks first in every reported setting, with $\mathrm{Acc}@0.9$ gains ranging from 0.41 points on RefCOCOg test to 7.02 points on RefCOCO+ testB. The larger gains at the strictest threshold are consistent with improved high-IoU localization and boundary alignment across the evaluated test splits.

\begin{table*}[ht]
\centering
\setlength{\belowcaptionskip}{5pt}
\caption{Localization under stricter IoU thresholds on five test splits. Split entries report $\mathrm{Acc}@0.7/\mathrm{Acc}@0.9$; aggregate metrics are computed jointly across the splits. $^{\dagger}$ denotes Hi-R1 aggregates reconstructed from its reported size groups; $^{\ddagger}$ marks locally reproduced baselines (Appendix~\ref{app:prompt_protocol}). Bold and underline denote the best and second-best results.}
\label{tab:stringent_iou_v4}

{\scriptsize
\setlength{\tabcolsep}{1.6pt}
\renewcommand{\arraystretch}{1.0}

\begin{tabular*}{\textwidth}{@{\extracolsep{\fill}}lcccccccc@{}}
\toprule
& \multicolumn{2}{c}{RefCOCO}
& \multicolumn{2}{c}{RefCOCO+}
& RefCOCOg
& \multicolumn{3}{c}{Aggregate Results} \\
\cmidrule(lr){2-3}\cmidrule(lr){4-5}\cmidrule(lr){6-6}
\cmidrule(lr){7-9}
Method
& testA
& testB
& testA
& testB
& test
& mIoU $\uparrow$
& Acc@0.7 $\uparrow$
& Acc@0.9 $\uparrow$ \\
\midrule

Hi-R1$^{\dagger}$~\citep{zhu2026hitoken}
& -
& -
& -
& -
& -
& 79.97
& 80.07
& 54.84 \\

InternVL3.5-8B$^{\ddagger}$~\citep{wang2025internvl35}
& 85.9/63.0
& 77.5/55.0
& 83.9/61.6
& 70.6/50.0
& 79.3/58.6
& 79.33
& 79.68
& 58.02 \\

Qwen3-VL-4B~\citep{bai2025qwen3vl,zhu2026ioupd}
& 88.6/--
& 81.3/--
& 86.0/--
& 73.3/--
& 82.2/--
& 81.74
& 82.51
& 61.24 \\

Qwen3-VL-8B$^{\ddagger}$~\citep{bai2025qwen3vl}
& 88.8/66.3
& 81.1/58.4
& 86.4/64.9
& 74.8/54.0
& 83.0/61.7
& 82.39
& 83.09
& 61.35 \\

Qwen3.5-9B$^{\ddagger}$~\citep{qwen2026qwen35}
& 88.7/65.9
& 81.7/58.6
& 86.4/64.5
& 75.6/54.3
& 83.1/59.6
& 82.16
& 83.31
& 60.64 \\

CogVLM-Grounding-17B$^{\ddagger}$~\citep{wang2024cogvlm}
& \underline{90.6/69.2}
& \underline{83.0/61.3}
& \underline{88.6/67.9}
& \underline{77.3/57.7}
& \underline{85.5/64.7}
& \underline{83.60}
& \underline{85.31}
& \underline{64.45} \\

\midrule

\textbf{CoEvolve (ours)}
& \textbf{92.7/74.2}
& \textbf{87.1/67.8}
& \textbf{90.7/73.1}
& \textbf{82.3/64.7}
& \textbf{86.9/65.2}
& \textbf{85.72}
& \textbf{87.94}
& \textbf{68.64} \\

\bottomrule
\end{tabular*}
}
\end{table*}

\subsection{Controlled Box Correction}
\label{sec:controlled_correction_v3}

Grounding self-correction has been studied through iterative prompting and visual verification~\citep{liao2025correct}. We evaluate a trained coordinate editor on 30,969 examples using the same deterministic corrupted anchors for all methods. The anchors span multiple IoU levels and geometric perturbations and have 58.80\% mIoU; Appendix~\ref{app:strict_and_controlled} gives the full construction protocol. This setting bypasses RER and directly evaluates end-to-end box correction under a shared corrupted-anchor protocol. BDR performs one same-position refinement pass, while the compared baselines generate corrected boxes from the same image, referring expression, and input anchor. Exact state templates, prompts, and parsing rules appear in Appendix~\ref{app:prompt_protocol}.

We report output localization metrics and $\mathrm{Repair}@t$, the fraction of initially sub-threshold anchors corrected to at least $t$:
\begin{equation}
\mathrm{Repair}@t =
\frac{\left|\{i:u_i<t,\ \hat{u}_i\geq t\}\right|}
     {\left|\{i:u_i<t\}\right|},
\label{eq:repair_metric_v3}
\end{equation}
where $u_i$ and $\hat{u}_i$ denote input and output IoUs.
We additionally report the mean IoU change,
$N^{-1}\sum_i(\hat{u}_i-u_i)$, and the improvement and degradation rates,
defined as the fractions of anchors whose output IoU is strictly higher and
lower than the corresponding input IoU, respectively.

Fig.~\ref{fig:correction} shows that BDR raises mIoU to 86.19\% and achieves 87.99\% $\mathrm{Repair}@0.7$. Among the compared correction baselines, Qwen3.5-9B is strongest in output mIoU. CoEvolve BDR improves output mIoU, $\mathrm{Acc}@0.9$, and $\mathrm{Repair}@0.7$ over this baseline by 2.34, 5.46, and 8.11 points, respectively. The degradation rate is 11.42\% for Qwen3.5-9B and 15.50\% for CoEvolve BDR.

\begin{figure*}[ht]
\centering
\includegraphics[width=\linewidth]{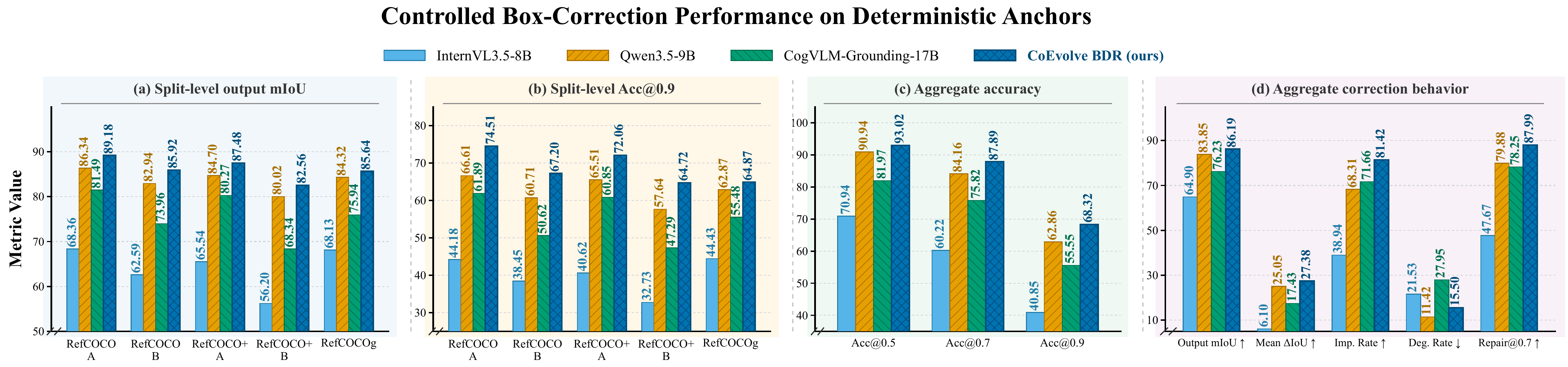}
\caption{Controlled box correction on five test splits with 30,969 deterministic anchors. Aggregate metrics are computed jointly over all anchors. CoEvolve BDR achieves the highest output mIoU, aggregate localization accuracies, mean IoU gain, improvement rate, and $\mathrm{Repair}@0.7$, while Qwen3.5-9B yields the lowest degradation rate. Higher values are better except for degradation rate.}
\label{fig:correction}
\end{figure*}

\subsection{State Construction and Source-Matched Refinement}
\label{sec:ablation_v3}

\paragraph{Meaningful-state ablation.}
Zero-Anchor BDR removes both the upstream localization hypothesis and the anchor-aware preservation and preference objectives enabled by that hypothesis, reducing training to direct same-position coordinate prediction. CoEvolve instead uses the RER state as both a localization starting point and an edit reference. Fig.~\ref{fig:zero_ablation} shows a 2.22-point mIoU gain, while the improvement grows from 2.19 points at $\mathrm{Acc}@0.5$ to 3.68 at $\mathrm{Acc}@0.7$ and 6.48 at $\mathrm{Acc}@0.9$. This widening gap supports the joint value of a meaningful prior localization state and the anchor-aware objectives it enables. Appendix Table~\ref{tab:app_ablation_split} reports split-level results.

\begin{figure*}[ht]
\centering
\includegraphics[width=\linewidth]{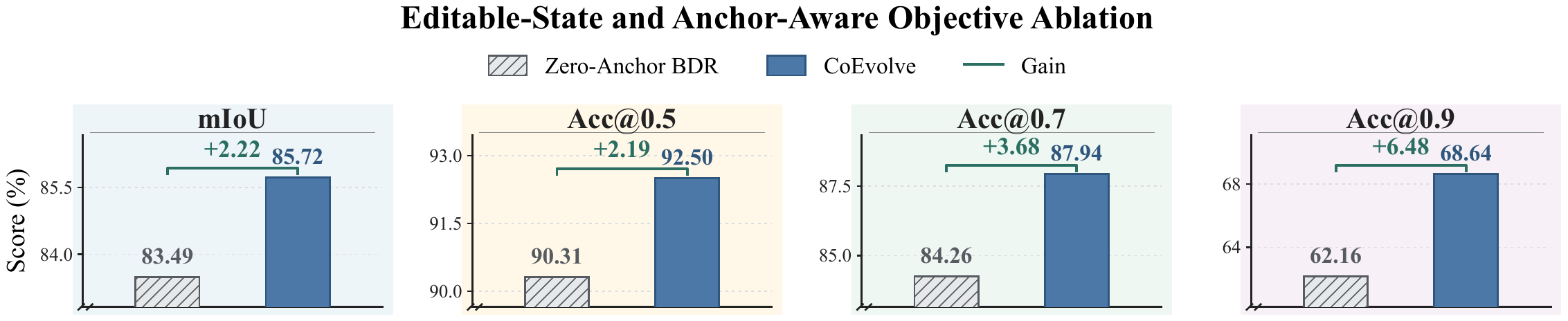}
\caption{Effect of a meaningful editable state and its associated anchor-aware objectives on five test splits. Aggregate metrics are computed jointly across the splits. Zero-Anchor BDR (ours) uses zero anchors; CoEvolve (ours) uses RER states and anchor-aware objectives. Gains are computed from unrounded aggregate values before display rounding.}
\label{fig:zero_ablation}
\end{figure*}

\paragraph{Structured state construction and source-matched refinement.}
Under the same structured interface, RER raises mIoU from 63.48\% to 80.70\% and $\mathrm{Acc}@0.9$ from 24.28 to 59.81 relative to unadapted Qwen3.5-9B. Because both policies use the same response format, these gains show the benefit of RER training for structured region evolution. For each state source, we train a separate BDR from scratch on its outputs. Table~\ref{tab:state_source_v3} shows that native outputs start at 82.16\% mIoU, above RER states at 80.70\%. After source-matched refinement, the RER pipeline gains 5.02 points and reaches 85.72\%, while the native-output pipeline gains 2.41 points and reaches 84.57\%. At $\mathrm{Acc}@0.9$, their respective gains are 8.83 and 5.68 points, yielding 68.64\% and 66.32\%. This reversal shows that initial localization accuracy alone does not determine how effectively a state can be refined.

\begin{table*}[!ht]
\centering
\setlength{\belowcaptionskip}{5pt}
\caption{Source-matched state construction and refinement on five test splits. Each source is paired with a separately trained BDR; metrics are computed jointly across the splits.}
\label{tab:state_source_v3}
\scriptsize
\setlength{\tabcolsep}{2.0pt}
\begin{tabular*}{\textwidth}{@{\extracolsep{\fill}}lccccccc@{}}
\toprule
State / editor & RER Trained? & Init. mIoU & Final mIoU & $\Delta$ mIoU & Init. Acc@0.9 & Final Acc@0.9 & $\Delta$ Acc@0.9 \\
\midrule
Structured Qwen3.5 & No & 63.48 & -- & -- & 24.28 & -- & -- \\
Native output $\rightarrow$ BDR (ours) & No & \textbf{82.16} & \underline{84.57} & \underline{+2.41} & \textbf{60.64} & \underline{66.32} & \underline{+5.68} \\
RER $\rightarrow$ BDR (ours) & Yes & \underline{80.70} & \textbf{85.72} & \textbf{+5.02} & \underline{59.81} & \textbf{68.64} & \textbf{+8.83} \\
\bottomrule
\end{tabular*}
\end{table*}

\subsection{Cross-Domain Adaptation and Downstream Segmentation}
\label{sec:remote_v3}

\paragraph{Remote-sensing grounding.}
After domain-specific retraining on DIOR-RSVG, CoEvolve achieves 86.20\% $\mathrm{Pr}@0.5$ and 76.23\% mIoU. Relative to the reproduced native Qwen3.5-9B baseline, which is evaluated without DIOR-RSVG adaptation, the complete domain-adapted system improves these metrics by 30.35 and 24.40 points, respectively. Table~\ref{tab:dior_v3} compares CoEvolve with task-specific methods and general-purpose MLLM baselines.

\begin{table}[!ht]
\centering
\setlength{\belowcaptionskip}{5pt}
\caption{Remote-sensing grounding on DIOR-RSVG. Public methods follow their
reported protocols; RER and BDR are retrained on the DIOR-RSVG training
split. $^{\dagger}$ denotes our reproduced native baseline evaluated without
DIOR-RSVG adaptation. Best and second-best results are bolded and underlined,
respectively; dashes denote unreported results.}
\label{tab:dior_v3}
\scriptsize
\setlength{\tabcolsep}{2.0pt}
\begin{tabular*}{\linewidth}{@{\extracolsep{\fill}}lcccc@{}}
\toprule
Method & $\mathrm{Pr}@0.5$ $\uparrow$ & $\mathrm{Pr}@0.7$ $\uparrow$ & $\mathrm{Pr}@0.9$ $\uparrow$ & mIoU $\uparrow$ \\
\midrule
LQVG~\citep{lan2024lqvg} & \underline{83.41} & \underline{75.91} & \textbf{43.53} & \underline{74.02} \\
LPVA~\citep{li2024lpva} & 82.27 & 72.25 & 39.55 & 72.35 \\
Qwen3.5-4B~\citep{qwen2026qwen35,wang2026geosearcher} & 53.10 & 39.40 & -- & 49.00 \\
InternVL3.5-8B~\citep{wang2025internvl35,wang2026geosearcher} & 64.90 & 50.50 & -- & 56.90 \\
RSGround-R1~\citep{huang2026rsgroundr1} & 71.80 & 58.70 & -- & 63.40 \\
GeoSearcher~\citep{wang2026geosearcher} & 83.20 & 72.40 & -- & 73.50 \\
Qwen3.5-9B$^{\dagger}$~\citep{qwen2026qwen35} & 55.85 & 43.15 & 18.81 & 51.83 \\
\midrule
CoEvolve (ours) & \textbf{86.20} & \textbf{76.16} & \underline{43.45} & \textbf{76.23} \\
\bottomrule
\end{tabular*}
\end{table}

\paragraph{Downstream segmentation.}
CoEvolve predictions also serve as box prompts to frozen SAM3~\citep{carion2026sam3}, yielding 78.19 Macro-8 cIoU across the eight RefCOCO-family splits without mask fine-tuning. Appendix~\ref{app:segmentation} provides the full split-wise comparison.

\section{Conclusion}

CoEvolve reformulates visual grounding as a construct-to-edit process: RER externalizes measurable region commitments, and BDR refines their coordinates through bidirectional same-position reconstruction with geometry-level and behavior-level supervision. Experiments across natural images, controlled perturbations, and remote-sensing scenes support this decomposition. The state-source comparison further shows that refinement depends not only on initial localization quality, but also on how the editable state is constructed and how its editor is trained.

\bibliography{arxiv_2027}
\bibliographystyle{arxiv_2027}

\appendix
\clearpage
\section{Appendix Contents}
\begingroup
\normalsize
\setlength{\baselineskip}{13.5pt}
\setlength{\parskip}{0pt}
\newcommand{\appcontentsline}[2]{%
  \par\addvspace{0.55\baselineskip}%
  \noindent\makebox[2.4em][l]{\textbf{\ref*{#1}}}%
  \hyperref[#1]{\textbf{#2}}\dotfill\hyperref[#1]{\pageref*{#1}}\par}
\newcommand{\appcontentssubline}[2]{%
  \noindent\hspace*{1.5em}\makebox[3.6em][l]{\ref*{#1}}%
  \hyperref[#1]{#2}\dotfill\hyperref[#1]{\pageref*{#1}}\par}
\appcontentsline{app:method_details}{Supplementary Method Details}
\appcontentssubline{app:notation}{RER--BDR Interface}
\appcontentssubline{app:sequential_training}{Sequential Training and BDR Data}
\appcontentssubline{app:rer_details}{Region-Evolution Reinforcement}
\appcontentssubline{app:bdr_details}{BDR Refinement and Targets}
\appcontentssubline{app:inference}{Inference Procedure}
\appcontentssubline{app:prompt_protocol}{Prompts and Evaluation}
\appcontentssubline{app:hyperparameters}{Implementation Notes}
\appcontentsline{app:formal_properties}{Formal Properties}
\appcontentsline{app:additional_results}{Additional Experimental Results}
\appcontentssubline{app:dataset_accounting}{Dataset Accounting}
\appcontentssubline{app:stage1_results}{RER Localization State Quality}
\appcontentssubline{app:qualitative_results}{Qualitative Grounding Results}
\appcontentssubline{app:strict_and_controlled}{Controlled Correction Details}
\appcontentssubline{app:remote}{Remote-Sensing Grounding}
\appcontentssubline{app:segmentation}{Downstream Segmentation}
\appcontentsline{app:limitations}{Limitations}
\endgroup
\clearpage

\section{Supplementary Method Details}
\label{app:method_details}

This appendix gives the complete mathematical specification of RER and BDR.
We use the same notation as in the main paper and make
explicit which quantities are used only to construct training targets and
which quantities are visible to the model.

\subsection{Notation and the RER--BDR Interface}
\label{app:notation}

For an image--expression pair $(I,q)$, a box is represented in the normalized
integer coordinate system by
$\mathbf b=[x_1,y_1,x_2,y_2]\in[0,1000]^4$, with $x_1<x_2$ and $y_1<y_2$.
The annotated box is $\mathbf b^{*}$.  RER produces a response
\begin{equation}
y=\big\langle \structtag{refine}\;(r_1,\mathbf b_1),\ldots,
(r_K,\mathbf b_K)\;\structtag{/refine}\;
\structtag{answer}\,\mathbf b_{\mathrm{ans}}\,\structtag{/answer}\big\rangle ,
\end{equation}
where $K\in\{3,\ldots,5\}$ is the number of refinement steps.  BDR
receives this complete response, including its text and all coordinate
tokens, as an editable localization state.  RER and BDR are optimized
sequentially and communicate only through this generated state.

\subsection{Sequential Training and BDR Dataset Construction}
\label{app:sequential_training}

CoEvolve is trained sequentially rather than end to end.  We first optimize
RER on the RefCOCOg training set using the rewards defined below.  After RER
training is complete, its adapter is frozen and used to generate one complete
localization response for each RefCOCOg training expression.  Each saved record
pairs the original image and referring expression with the complete RER
assistant response and its annotated box.  The generated RER responses are then
used to construct the BDR training set.

BDR initializes a separate adapter on the frozen Qwen3.5-9B backbone; it is
not initialized from the RER adapter.  At loading time, auxiliary
\structtag{think} content is removed, every coordinate is zero-padded to four
digits, and the remaining \structtag{refine} and \structtag{answer} sequence is
used as the editable input state.  The unmodified records retain the natural
errors produced by RER.  Eligible high-overlap states are additionally
augmented with synchronized geometric perturbations, as detailed in
Section~\ref{app:bdr_details}, so BDR observes both the natural RER error
distribution and deliberately corrupted anchors.  The annotation is retained
only for target construction and loss computation and is never inserted into
the model input.

\subsection{Region-Evolution Reinforcement}
\label{app:rer_details}

\paragraph{Interface adaptation and state construction.}
RER adapts the backbone from terminal box generation to a structured interface
in which each reasoning Step exposes an explicit region.  The terminal
annotation evaluates both the final answer and properties of the generated
trajectory, deriving trajectory-level spatial rewards from terminal annotations without
introducing manually labeled rationales or intermediate boxes.  The resulting
response is therefore optimized as a localization state for subsequent
evaluation and editing, rather than as explanation text alone.

\paragraph{Structural rewards.}
For a non-empty target, let $n$ be the number of Step headers. The format reward is
\begin{equation}
R_{\mathrm{fmt}}=
\begin{cases}
0, & \text{invalid outer tagged structure},\\
0.5, & \text{valid outer structure and }(n<3\text{ or }n>5),\\
1, & \text{valid outer structure and }3\le n\le 5.
\end{cases}
\label{eq:app_format_reward}
\end{equation}
All evaluated benchmarks use non-empty targets; the empty-target branch is retained as a generic interface fallback. For an empty target, a valid tagged response with an empty answer receives
one, without a minimum Step count.  The step-format reward checks every
Step independently and is
\begin{equation}
R_{\mathrm{step}}=
\begin{cases}
R_{\mathrm{fmt}}, & \text{empty target},\\
R_{\mathrm{fmt}}\,\mathbf{1}\!\left[
\text{each Step has a parseable final bracketed box}\right],
& \text{otherwise}.
\end{cases}
\label{eq:app_step_reward}
\end{equation}
When a Step contains several bracketed expressions, the last one is used as
the Step box.  Thus, the reward evaluates the final state written by a Step
rather than an earlier box mentioned in its explanation.

\paragraph{Terminal box reward.}
For a valid single predicted box, let
$u=\operatorname{IoU}(\mathbf b_{\mathrm{ans}},\mathbf b^{*})$ and
$g=\operatorname{GIoU}(\mathbf b_{\mathrm{ans}},\mathbf b^{*})$.  The
generalized overlap is computed from the minimum enclosing rectangle
$C$ and its area $A_C$:
\begin{equation}
\operatorname{GIoU}=u-\frac{A_C-A_{\cup}}{A_C},
\end{equation}
where $A_{\cup}$ is the union area.  The scale and aspect agreement term is
\begin{equation}
S_A=\frac{1}{2}\left(
\frac{\min(A_{\mathrm{ans}},A^{*})}{\max(A_{\mathrm{ans}},A^{*})}
+\frac{\min(\rho_{\mathrm{ans}},\rho^{*})}
{\max(\rho_{\mathrm{ans}},\rho^{*})}\right),
\qquad \rho=\frac{\text{width}}{\text{height}}.
\label{eq:app_shape_score}
\end{equation}
The raw terminal score and its format-modulated form are
\begin{align}
R_{\mathrm{box}}^{\mathrm{raw}}
&=u+0.5u^2+0.2g+
0.2\,\sigma\!\left(10(u-0.5)\right)S_A, \label{eq:app_raw_bbox}\\
R_{\mathrm{box}}
&=R_{\mathrm{fmt}}\max(0,R_{\mathrm{box}}^{\mathrm{raw}})
 +\min(0,R_{\mathrm{box}}^{\mathrm{raw}}). \label{eq:app_bbox}
\end{align}
Here $\sigma(\cdot)$ denotes the logistic sigmoid.
The implementation assigns $-0.5$ to a malformed or missing answer, an
answer with the wrong number of boxes, or an answer containing multiple
boxes.  A correct empty-target answer receives one.  Negative geometry
feedback is never attenuated by an invalid format score.

\paragraph{Progressive reward.}
Let $u_i=\operatorname{IoU}(\mathbf b_i,\mathbf b^{*})$.  To provide a
distance signal before overlap is nonzero, the center quality of Step $i$ is
\begin{equation}
q_i=1-\frac{d_i^{\mathrm{norm}}}{\sqrt{2}},
\end{equation}
where $d_i^{\mathrm{norm}}$ is the Euclidean center distance after mapping
coordinates to $[0,1]^2$. Let $d_i^{\mathrm{center}}$ be the Euclidean
distance between the centers of $\mathbf b_i$ and $\mathbf b^{*}$ in the
original $[0,1000]^2$ coordinate system. The first Step receives a coverage score
\begin{equation}
C_{\mathrm{cov}}=\frac{1}{2}\left(C_{\mathrm{area}}+
\exp\left[-d_1^{\mathrm{center}}/350\right]\right).
\label{eq:app_coverage}
\end{equation}
Let $A_1$ and $A^{*}$ denote the first-Step and target areas, and set
$A_{\mathrm{req}}=\max(2A^{*},10^4)$.  The implemented area term is
\begin{equation}
C_{\mathrm{area}}=
\begin{cases}
1, & A_1\geq 5\times10^5,\\
\sigma\!\left(6(A_1/A_{\mathrm{req}}-1)\right),&\text{otherwise}.
\end{cases}
\label{eq:app_area_coverage}
\end{equation}
For later Steps, RER compares each box with the best preceding state. We use
$\epsilon=10^{-8}$ to stabilize the remaining-improvement denominators:
\begin{align}
\Delta_i^{u}
&=\frac{u_i-\max_{j<i}u_j}{1-u_1+\epsilon},&
\Delta_i^{q}
&=\frac{q_i-\max_{j<i}q_j}{1-q_1+\epsilon},\nonumber\\
g_i&=0.6\Delta_i^{u}+0.4\Delta_i^{q},&
C_{\mathrm{prog}}&=\max\left(0,\sum_{i=2}^{K}g_i\right).
\label{eq:app_progress}
\end{align}
The overall first-to-last gain and the last-Step/answer consistency are
\begin{equation}
C_{\mathrm{overall}}=\max\left(0,
\frac{u_K-u_1}{1-u_1+\epsilon}\right),\qquad
C_{\mathrm{cons}}=\operatorname{IoU}(\mathbf b_K,\mathbf b_{\mathrm{ans}}).
\label{eq:app_overall}
\end{equation}
The progressive reward and the complete RER reward are
\begin{align}
R_{\mathrm{prog}}
&=R_{\mathrm{fmt}}\left(
0.20C_{\mathrm{cov}}+0.50C_{\mathrm{prog}}
 +0.20C_{\mathrm{overall}}+0.10C_{\mathrm{cons}}\right), \label{eq:app_prog_reward}\\
R_{\mathrm{RER}}
&=0.5R_{\mathrm{fmt}}+0.5R_{\mathrm{step}}
 +R_{\mathrm{prog}}+R_{\mathrm{box}}. \label{eq:app_rer_total}
\end{align}
If the trajectory has fewer than three or more than five Steps, or any Step
does not contain a valid positive-area box, $R_{\mathrm{prog}}$ is zero.
In implementation, Eq.~(\ref{eq:app_rer_total}) is realized by supplying
$(R_{\mathrm{fmt}},R_{\mathrm{step}},R_{\mathrm{prog}},R_{\mathrm{box}})$
to GRPO with weights $0.5,0.5,1.0,1.0$, respectively; these weights are not
applied again to an already combined reward.
For each prompt, eight sampled responses form one GRPO group.  The
annotation is used to evaluate geometry, but no human-written reasoning
chain is used as a token target.

Figure~\ref{fig:rer_rewards} places the corresponding training diagnostics
beside the reward specification.  The curves show the total reward together
with the structural, progressive, and terminal geometry components.

\begin{figure}[!htb]
    \centering
    \paperfigure[width=0.238\textwidth]{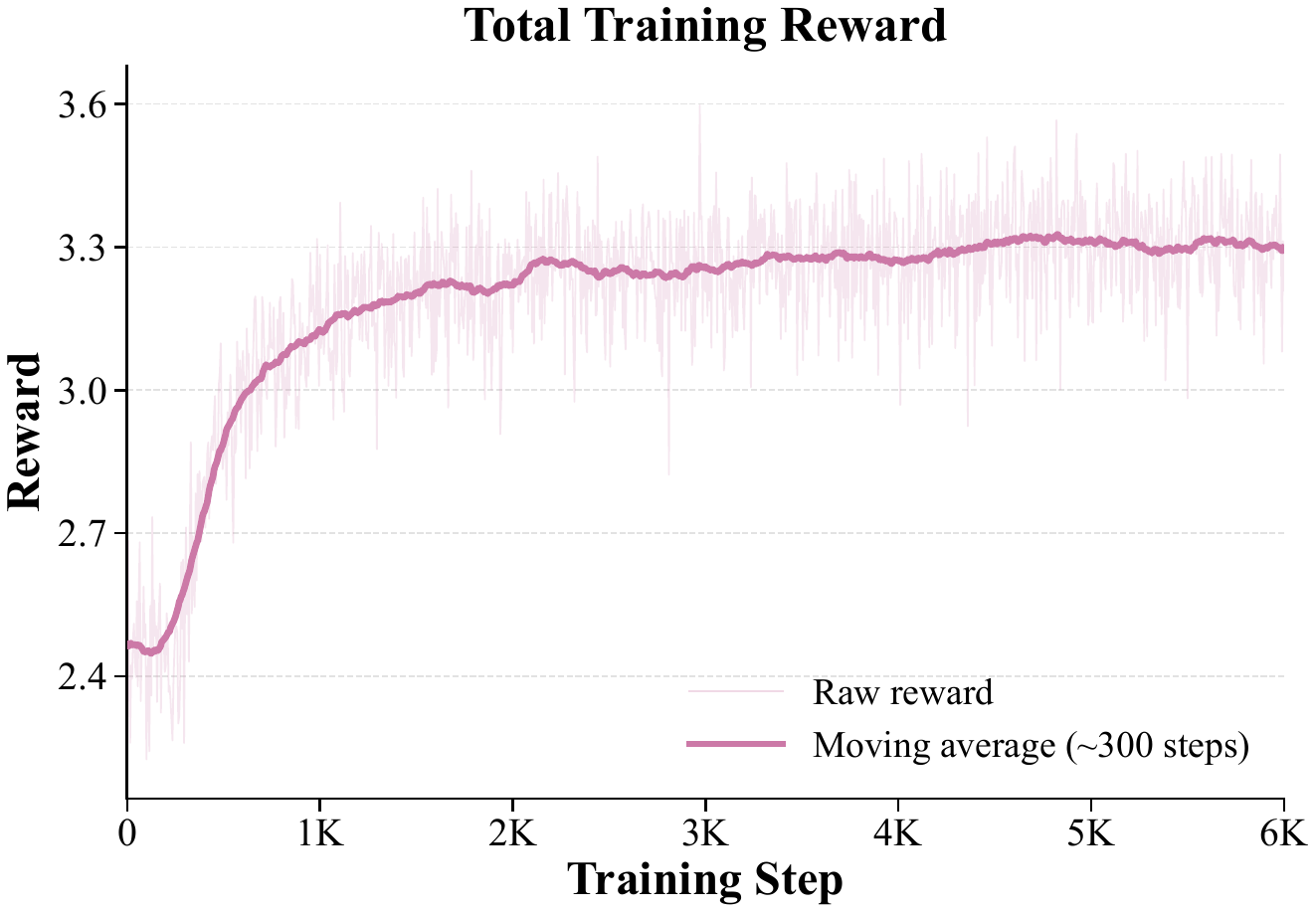}{0.225\textwidth}{25mm}\hfill
    \paperfigure[width=0.238\textwidth]{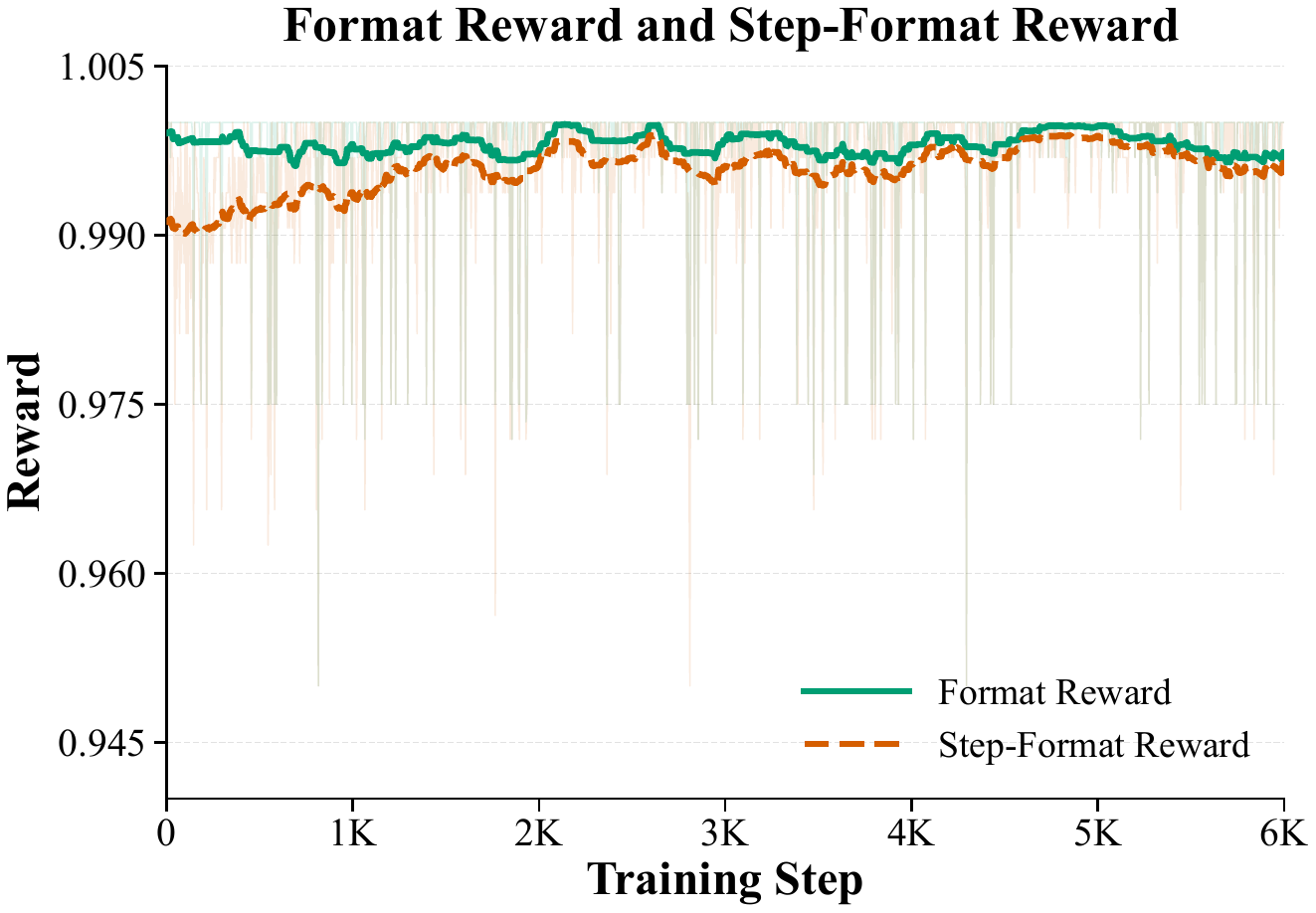}{0.225\textwidth}{25mm}\hfill
    \paperfigure[width=0.238\textwidth]{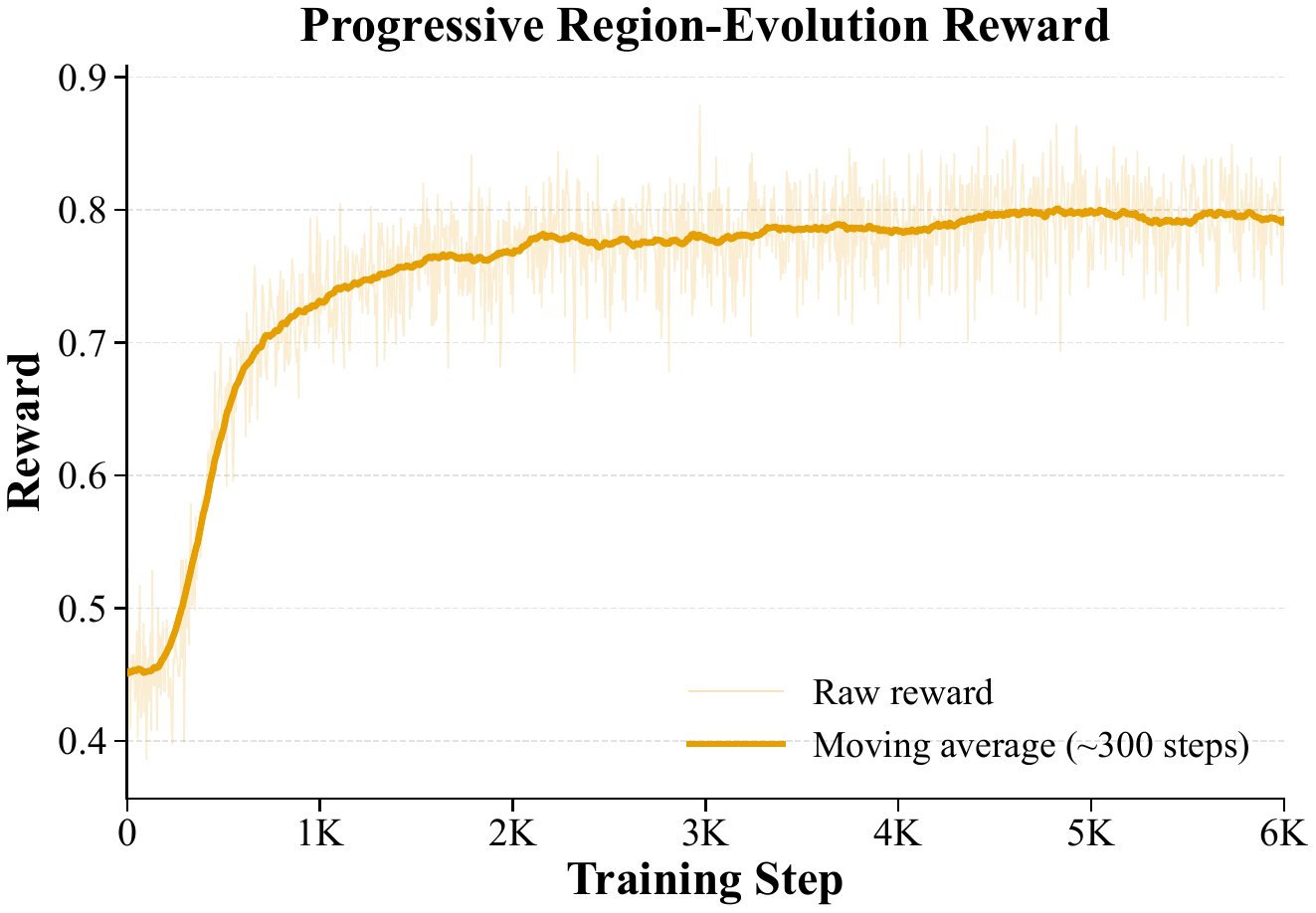}{0.225\textwidth}{25mm}\hfill
    \paperfigure[width=0.238\textwidth]{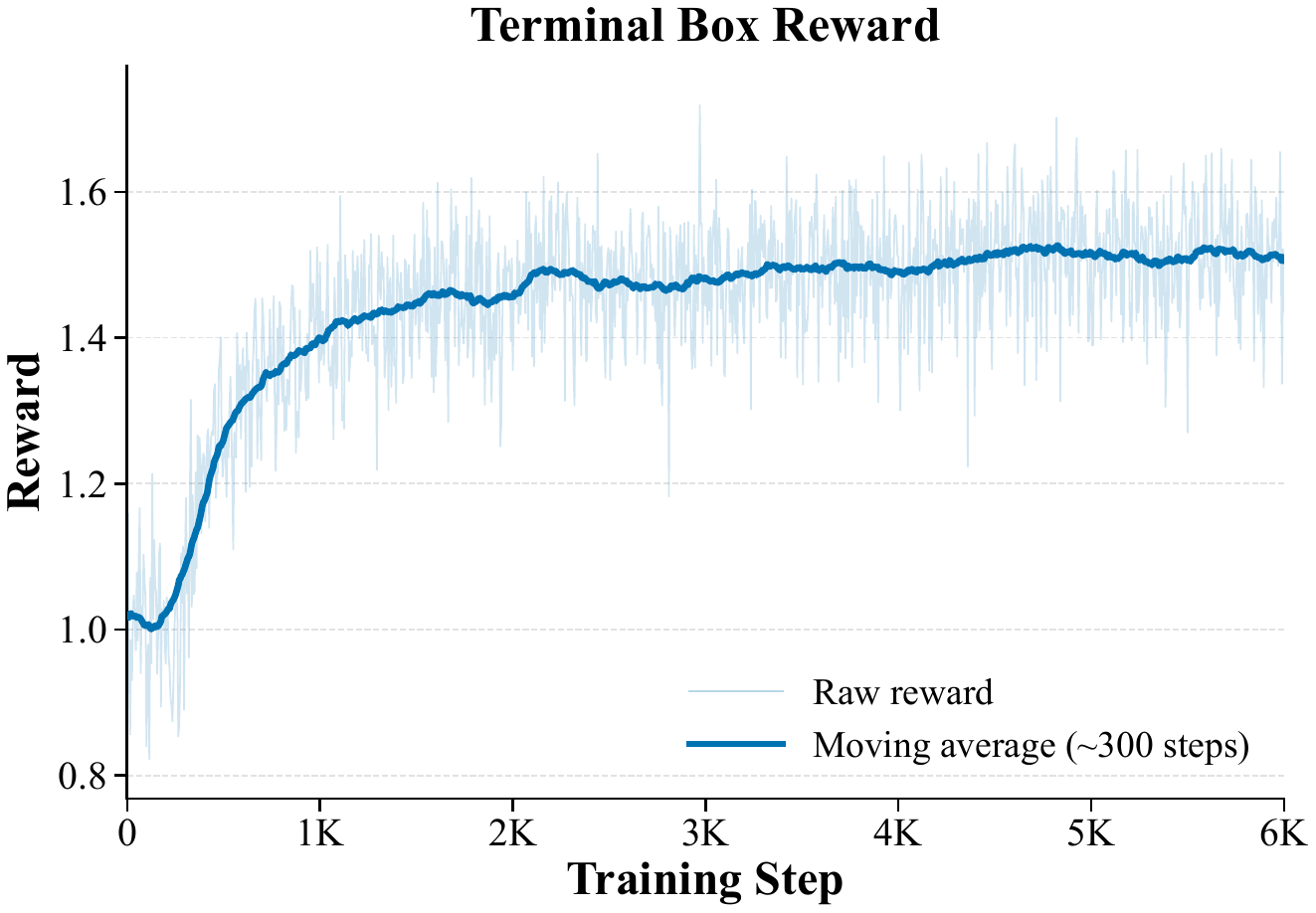}{0.225\textwidth}{25mm}
    \caption{Supplementary training dynamics for RER. For visualization, we show the first 6K optimization steps; this window does not indicate the total training duration. From left to right:
    total reward; format and Step-format rewards; progressive region-evolution
    reward; and terminal box reward.}
    \label{fig:rer_rewards}
\end{figure}

\subsection{Bidirectional Same-Position Refinement and Target Construction}
\label{app:bdr_details}

\paragraph{Editable state and masks.}
Every coordinate is written as four zero-padded digit tokens, so one box
contains sixteen digit positions.  Let $\mathcal M_{\mathrm{all}}$ be all
box digit positions in the assistant response, $\mathcal M_b$ the positions
of the last refinement box and the answer, and $\mathcal M_f$ the structural
tags and punctuation.  The normal strongly supervised region therefore has
$|\mathcal M_b|=32$.

BDR removes causal masking from the base model's Full Attention layers,
providing each same-position prediction with bidirectional access to the
global response context.  If $\mathbf z=F_{\phi}(I,q,\mathbf x^{\mathrm{in}})$
are logits aligned with the observed response, a digit candidate is
\begin{equation}
\widehat{x}_t=\arg\max_{v\in\mathcal V_{\mathrm{digit}}}z_{t,v},
\qquad t\in\mathcal M_{\mathrm{all}},
\end{equation}
where $\mathcal V_{\mathrm{digit}}$ contains the ten single-character
digits.  The 16 candidate digits associated with each box role are decoded
separately into one candidate box.

Same-position reconstruction predicts every editable coordinate from the same
observed response and retains alignment with its original response position.
Keeping the observed coordinate tokens visible provides an explicit input
reference against which candidate edits can be evaluated and accurate
coordinates can be preserved.
The decoded box therefore provides the synchronized candidate used by
coordinate regression, quality-gated target selection, entropy regularization,
and relative preference between the input and output.

\paragraph{Behavior-Level Supervision: Perturbation and Quality-Gated Targets.}
When the RER answer has IoU above $0.5$, it is selected for perturbation
with probability $0.3$.  The selected answer is replaced by a nonzero box
obtained through rejection sampling toward the configured IoU interval
$[0.2,0.5]$ relative to the annotated box. The sampler uses translation, shrink, expansion, or random
geometry and rejects zero-area boxes; after 200 unsuccessful attempts it
uses bounded coordinate offsets.  The answer, the last refinement box,
and earlier Steps whose IoU with the original answer exceeds $0.95$ are
replaced together to preserve spatial consistency among the perturbed
coordinate fields.

At each rejection-sampling attempt, the four proposal families are selected
with equal probability.  Their exact ranges are summarized in
Table~\ref{tab:app_training_perturbation}.  Every proposal is clipped to the
normalized coordinate range and accepted only when it has positive area and
$\operatorname{IoU}(\mathbf b_{\mathrm{proposal}},\mathbf b^{*})\in[0.2,0.5]$. The fallback independently offsets
the two horizontal coordinates within $\pm0.8$ target widths and the two
vertical coordinates within $\pm0.8$ target heights before clipping.

\begin{table}[!ht]
\centering
\setlength{\belowcaptionskip}{5pt}
\caption{Geometric proposals used to construct BDR training perturbations.
Probabilities refer to proposal selection at each rejection-sampling attempt;
the accepted proportions may differ because proposals are filtered by the
target IoU interval $[0.2,0.5]$.}
\label{tab:app_training_perturbation}
\small
\setlength{\tabcolsep}{4.0pt}
\begin{tabular}{lcl}
\toprule
Proposal & Probability & Parameter range \\
\midrule
Coordinate shift & $0.25$ & $s\sim\mathcal U(0.2,1.5)$; offsets in $[-sw,sw]$ and $[-sh,sh]$ \\
Centered shrink & $0.25$ & width and height ratio $r\sim\mathcal U(0.3,0.9)$ \\
Centered expansion & $0.25$ & width and height ratio $r\sim\mathcal U(1.3,3.0)$ \\
Random box & $0.25$ & valid integer endpoints in $[0,1000]$ \\
\bottomrule
\end{tabular}
\end{table}

In the coordinate-shift proposal, $w$ and $h$ denote the width and height of
the ground-truth box, and $s$ is the sampled offset-scale factor.

Let $\mathbf b_{\mathrm{ans}}^{\mathrm{in}}$ denote the actual Answer box seen by BDR after this
operation, and let $\widetilde{\mathbf b}_{\mathrm{ans}}$ be the hard
candidate decoded from its output.  With $\eta=0.9$ and
$\gamma=0.005$, the dynamic-copy condition is
\begin{equation}
\mathcal C_{\mathrm{out}}\equiv
\left(\operatorname{IoU}(\widetilde{\mathbf b}_{\mathrm{ans}},\mathbf b^{*})
>\eta\right)
\land
\left(\operatorname{IoU}(\widetilde{\mathbf b}_{\mathrm{ans}},\mathbf b^{*})
\geq \operatorname{IoU}(\mathbf b_{\mathrm{ans}}^{\mathrm{in}},\mathbf b^{*})+\gamma\right).
\end{equation}
If it is not activated and the unperturbed input has IoU at least $\eta$,
the static preservation branch is used.  The selected absolute target is
\begin{equation}
\mathbf b^{\mathrm{tar}}=
\begin{cases}
\widetilde{\mathbf b}_{\mathrm{ans}}, & \text{if }\mathcal C_{\mathrm{out}},\\
\mathbf b_{\mathrm{ans}}^{\mathrm{in}},&
\text{unperturbed and }\operatorname{IoU}(\mathbf b_{\mathrm{ans}}^{\mathrm{in}},\mathbf b^{*})\geq\eta,\\
\mathbf b^{*},&\text{otherwise}.
\end{cases}
\label{eq:app_target}
\end{equation}
The ground-truth box is used to construct this training target and the
geometry labels only; it is never provided as an input during inference.

\paragraph{Geometry-Level Supervision.}
For a target digit $y_t$, cross-entropy over positions $\mathcal M$ is
\begin{equation}
\mathcal L_{\mathrm{CE}}(\mathcal M)
=-\frac{1}{|\mathcal M|}\sum_{t\in\mathcal M}
\log\frac{\exp z_{t,y_t}}{\sum_v\exp z_{t,v}}.
\label{eq:app_masked_ce}
\end{equation}
Let $\widehat b^{\mathrm{hard}}_{sk}$ denote coordinate $k$ of the hard
argmax-decoded candidate box for role $s$. For the final Step and answer,
the four digits belonging to coordinate $k$ share one detached weight
obtained by averaging the two hard-coordinate errors:
\begin{equation}
w_k=1+\frac{w_{\mathrm{L1CE}}}{2}
\sum_{s\in\{\mathrm{last},\mathrm{ans}\}}
\frac{\left|\widehat b^{\mathrm{hard}}_{sk}-b^{\mathrm{tar}}_k\right|}{1000},
\qquad w_{\mathrm{L1CE}}=10.
\label{eq:app_l1_ce_weight}
\end{equation}
Each digit position $t$ of coordinate $k$ inherits $w_t=w_k$, and $\mathrm{CE}_t$ denotes its token cross-entropy. The weighted BBox CE divides by the sum of token weights:
\begin{equation}
\mathcal L_{\mathrm{box}}
=\frac{\sum_{t\in\mathcal M_b}w_t\,\mathrm{CE}_t}
{\sum_{t\in\mathcal M_b}w_t}.
\label{eq:app_weighted_box_ce}
\end{equation}
Consequently, L1-derived weights redistribute the gradient toward inaccurate
coordinates without changing the nominal CE scale.

The differentiable coordinate used by the regression term is obtained from
the ten-digit softmax at each digit position:
\begin{align}
p_{skj}(d)&=
\frac{\exp z_{t_{skj},d}}
{\sum_{a=0}^{9}\exp z_{t_{skj},a}},\nonumber\\
\bar b_{sk}&=\sum_{j=0}^{3}10^{3-j}\sum_{d=0}^{9}d\,p_{skj}(d),\nonumber\\
\mathcal L_{\mathrm{reg}}
&=\frac{1}{8}\sum_{s\in\{\mathrm{last},\mathrm{ans}\}}
\sum_{k=1}^{4}\frac{|\bar b_{sk}-b^{\mathrm{tar}}_k|}{1000}.
\label{eq:app_regression}
\end{align}
The code computes the same quantity as the mean over the eight supervised
coordinates.  CE preserves the discrete coordinate representation, whereas
the regression term supplies a continuous numerical direction.

\paragraph{State Preservation and Confidence Regularization.}
General reconstruction is applied to assistant tokens outside the coordinate
mask.  In the warm-up phase it excludes only $\mathcal M_b$, so a perturbed
historical Step can still receive the selected target through general CE.
After warm-up it excludes $\mathcal M_{\mathrm{all}}$ and therefore preserves
only non-coordinate assistant content.  Format CE is computed on
$\mathcal M_f$, including structural tags and coordinate punctuation.

For the ten-digit distribution $p_t$, natural entropy is
\begin{equation}
H(p_t)=-\sum_{d=0}^{9}p_t(d)\log p_t(d).
\label{eq:app_entropy}
\end{equation}
The historical entropy term averages $H(p_t)$ over
$\mathcal M_{\mathrm{all}}\setminus\mathcal M_b$.  The final entropy term
uses $\mathcal M_b$ but keeps only positions whose digit-vocabulary argmax
already equals the selected target digit.  Thus the former sharpens
unsupervised historical digits, while the latter does not make an incorrect
final digit more confident.

\paragraph{Coordinate-wise Edit Preference.}
For a fixed sample $i$ and every box role $s$ (all historical Steps, the last Step, and the
answer), let $\mathcal C_{sk}$ be the changed digit positions in coordinate
$k$. The input and output coordinate errors are compared to the original
annotation:
\begin{equation}
\Delta_{sk}=
\frac{|\breve b^{\mathrm{in}}_{sk}-b_k^{*}|-
|\widetilde b^{\mathrm{out}}_{sk}-b_k^{*}|}{1000}.
\label{eq:app_delta}
\end{equation}
Only coordinates with $\mathcal C_{sk}\neq\varnothing$ and
$|\Delta_{sk}|>m$ with $m=0.005$ are active; their index set is
$\mathcal A_s=\{k:\mathcal C_{sk}\neq\varnothing,\ |\Delta_{sk}|>m\}$.
Let $x_t^{\mathrm{in}}$ denote
the observed input digit token and $\widetilde{x}_t^{\mathrm{out}}$ the
digit-vocabulary argmax at position $t$. On changed digit positions,
\begin{equation}
D_{sk}=\frac{1}{|\mathcal C_{sk}|}
\sum_{t\in\mathcal C_{sk}}
\left(z_{t,\widetilde{x}_t^{\mathrm{out}}}
-z_{t,x_t^{\mathrm{in}}}\right),\qquad
y_{sk}=\mathbf{1}[\Delta_{sk}>m].
\label{eq:app_pref_margin}
\end{equation}
The coordinate loss is the expanded binary cross-entropy with logits:
\begin{equation}
\ell_{sk}
=-y_{sk}\log\sigma(D_{sk})
-(1-y_{sk})\log\left(1-\sigma(D_{sk})\right),
\qquad \sigma(D)=\frac{1}{1+\exp(-D)}.
\label{eq:app_pref_bce}
\end{equation}
For $\mathcal A_s\neq\varnothing$, active coordinates within one role are weighted by
$a_{sk}=\min(1,|\Delta_{sk}|/0.05)$:
\begin{equation}
\mathcal L_{\mathrm{pref},s}^{(i)}
=\frac{\sum_{k\in\mathcal A_s}a_{sk}\ell_{sk}}
{\sum_{k\in\mathcal A_s}a_{sk}}.
\end{equation}
We set $\mathcal L_{\mathrm{pref},s}^{(i)}=0$ when $\mathcal A_s=\varnothing$.
The answer and the last Step have role weight $0.5$, and every earlier Step
has role weight $0.2$.  The per-sample and minibatch losses are
\begin{equation}
\mathcal L_{\mathrm{pref}}^{(i)}=
\sum_{s\in\mathcal S}w_s\mathcal L_{\mathrm{pref},s}^{(i)},
\qquad
\mathcal L_{\mathrm{pref}}
=\frac{1}{N}\sum_{i=1}^{N}
\mathcal L_{\mathrm{pref}}^{(i)}.
\label{eq:app_pref_total}
\end{equation}
The preference branches are added directly to the base loss.  Hence a sample
can contribute one active preference term per role, while coordinates inside
each role are normalized by their own active magnitude weights.  Samples
without an active preference coordinate contribute zero but remain in the
minibatch denominator $N$.

\paragraph{Phase-specific objectives.}
During the first $3\%$ of optimizer updates, the training loss is
\begin{equation}
\mathcal L_{\mathrm{early}}
=1.0\,\mathcal L_{\mathrm{box}}
 +0.2\,\mathcal L_{\mathrm{gen}}.
\label{eq:app_early_loss}
\end{equation}
Perturbed historical Step labels are covered by the general CE mask in this
phase, and $10\%$ of eligible assistant input tokens are replaced by the
configured mask token.  No token masking is applied after warm-up.  BDR then
uses
\begin{align}
\mathcal L_{\mathrm{mid}}
={}&1.0\,\mathcal L_{\mathrm{box}}
 +30\,\mathcal L_{\mathrm{reg}}
 +0.5\,\mathcal L_{\mathrm{fmt}}
 +0.2\,\mathcal L_{\mathrm{gen}}\nonumber\\
&+0.2\,\mathcal H_{\mathrm{hist}}
 +0.2\,\mathcal H_{\mathrm{fin}}
 +\mathcal L_{\mathrm{pref}}.
\label{eq:app_mid_loss}
\end{align}
Only the last Step and answer enter $\mathcal L_{\mathrm{box}}$,
$\mathcal L_{\mathrm{reg}}$, and $\mathcal H_{\mathrm{fin}}$; historical
boxes enter $\mathcal H_{\mathrm{hist}}$.  The total sample losses are
divided by the effective minibatch size before back-propagation.

Figure~\ref{fig:bdr_losses} follows the complete BDR objective and reports
the optimization dynamics of its coordinate, reconstruction, confidence, and
preference terms.

\begin{figure}[!htb]
    \centering
    \paperfigure[width=0.238\textwidth]{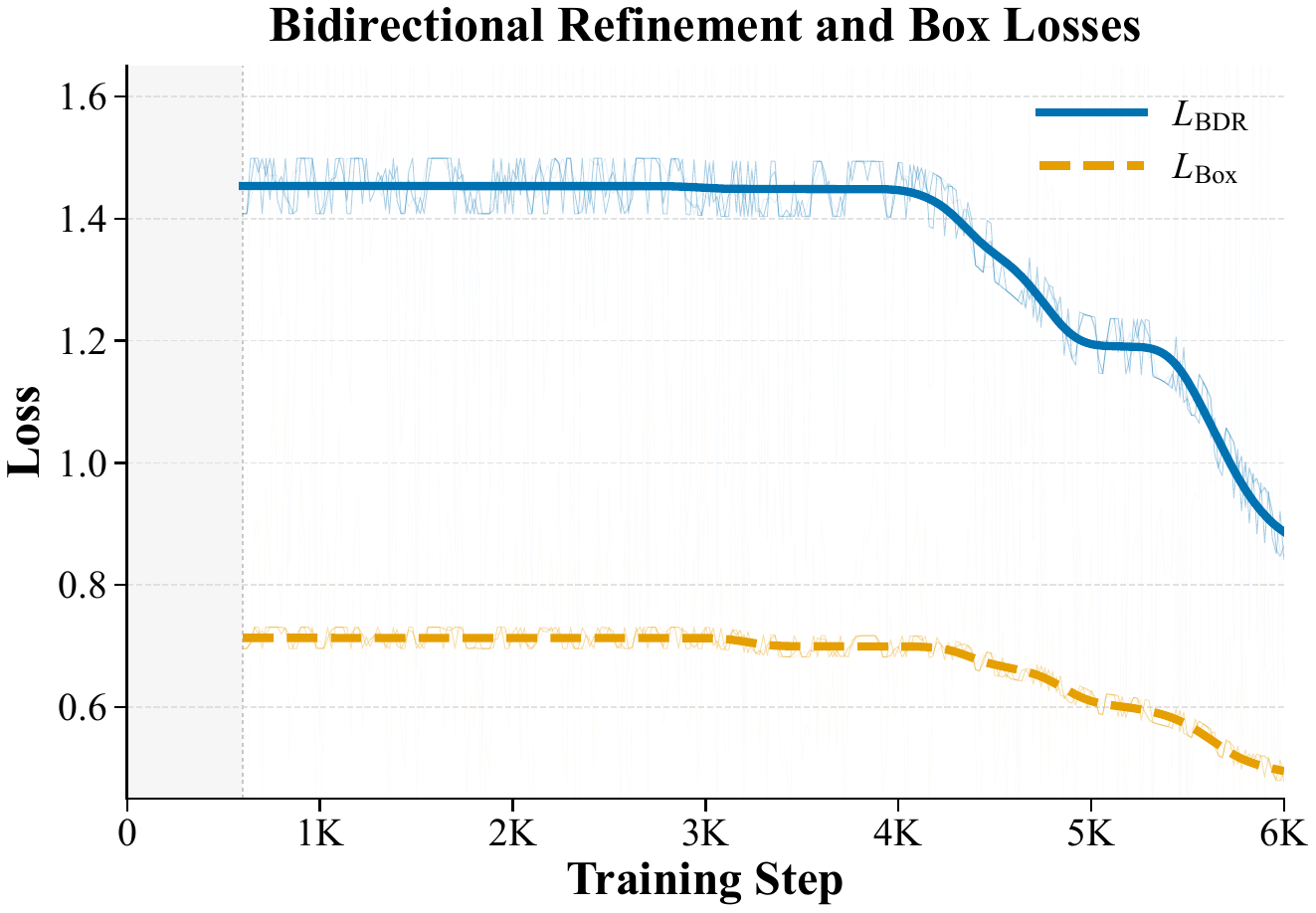}{0.225\textwidth}{25mm}\hfill
    \paperfigure[width=0.238\textwidth]{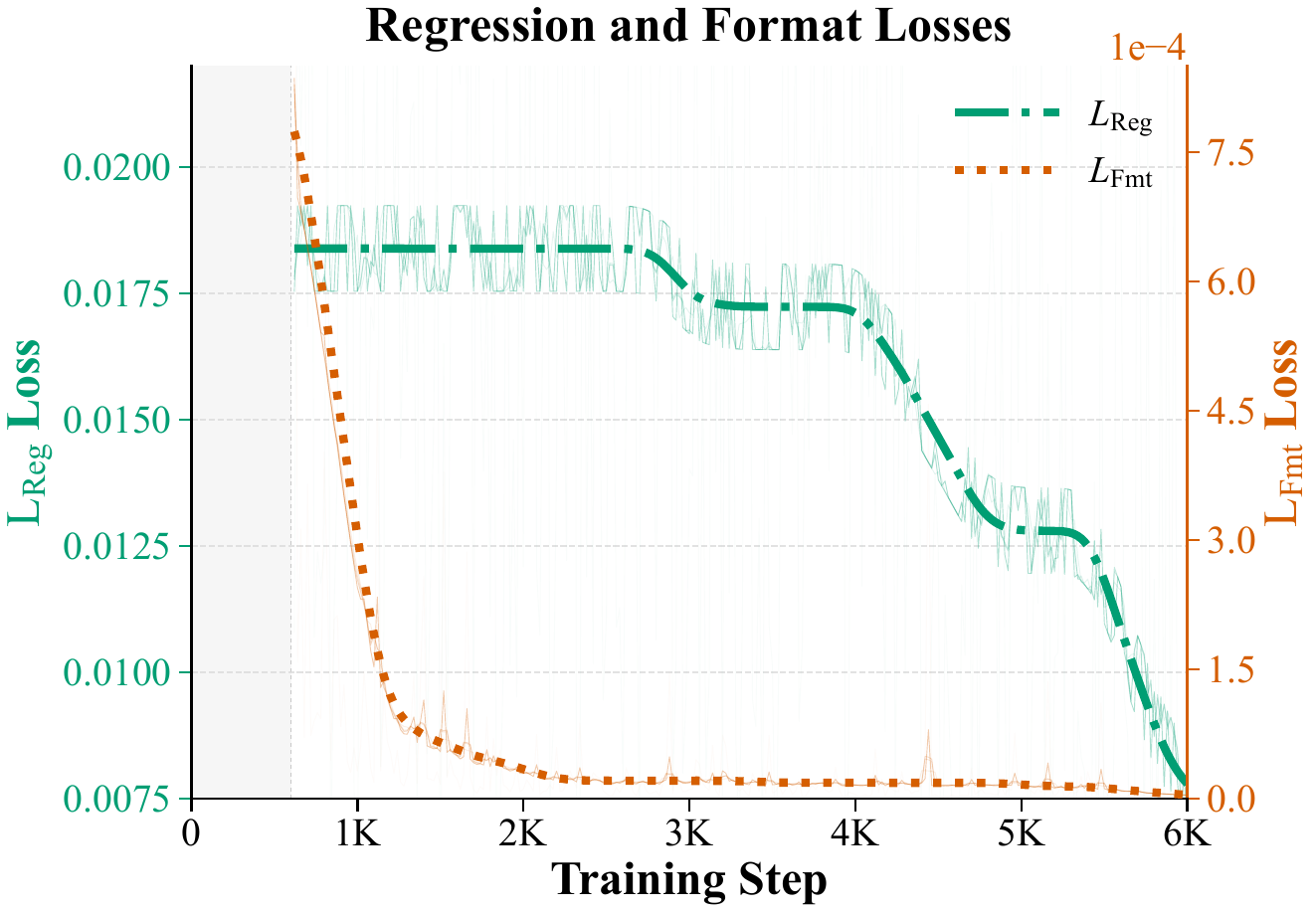}{0.225\textwidth}{25mm}\hfill
    \paperfigure[width=0.238\textwidth]{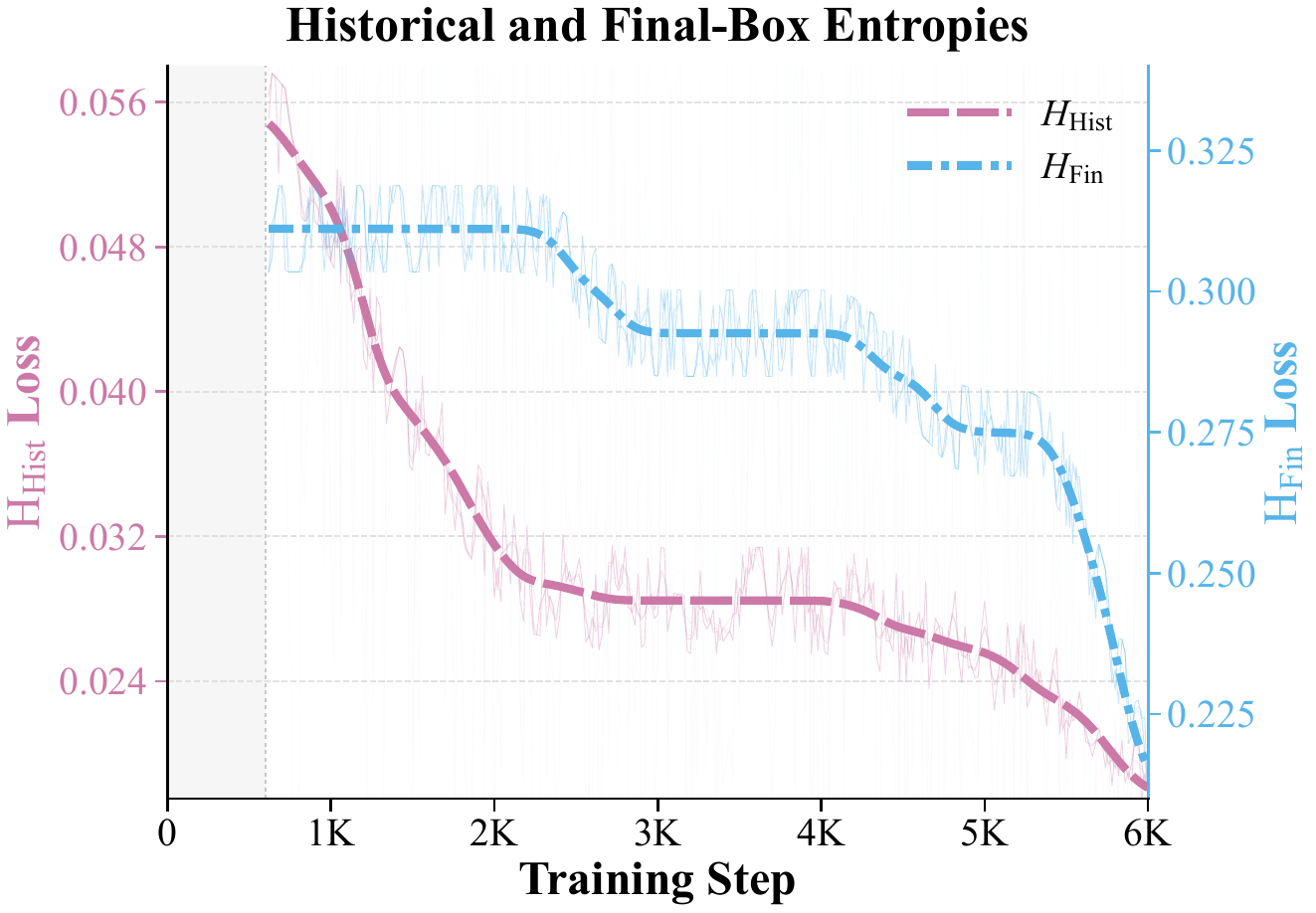}{0.225\textwidth}{25mm}\hfill
    \paperfigure[width=0.238\textwidth]{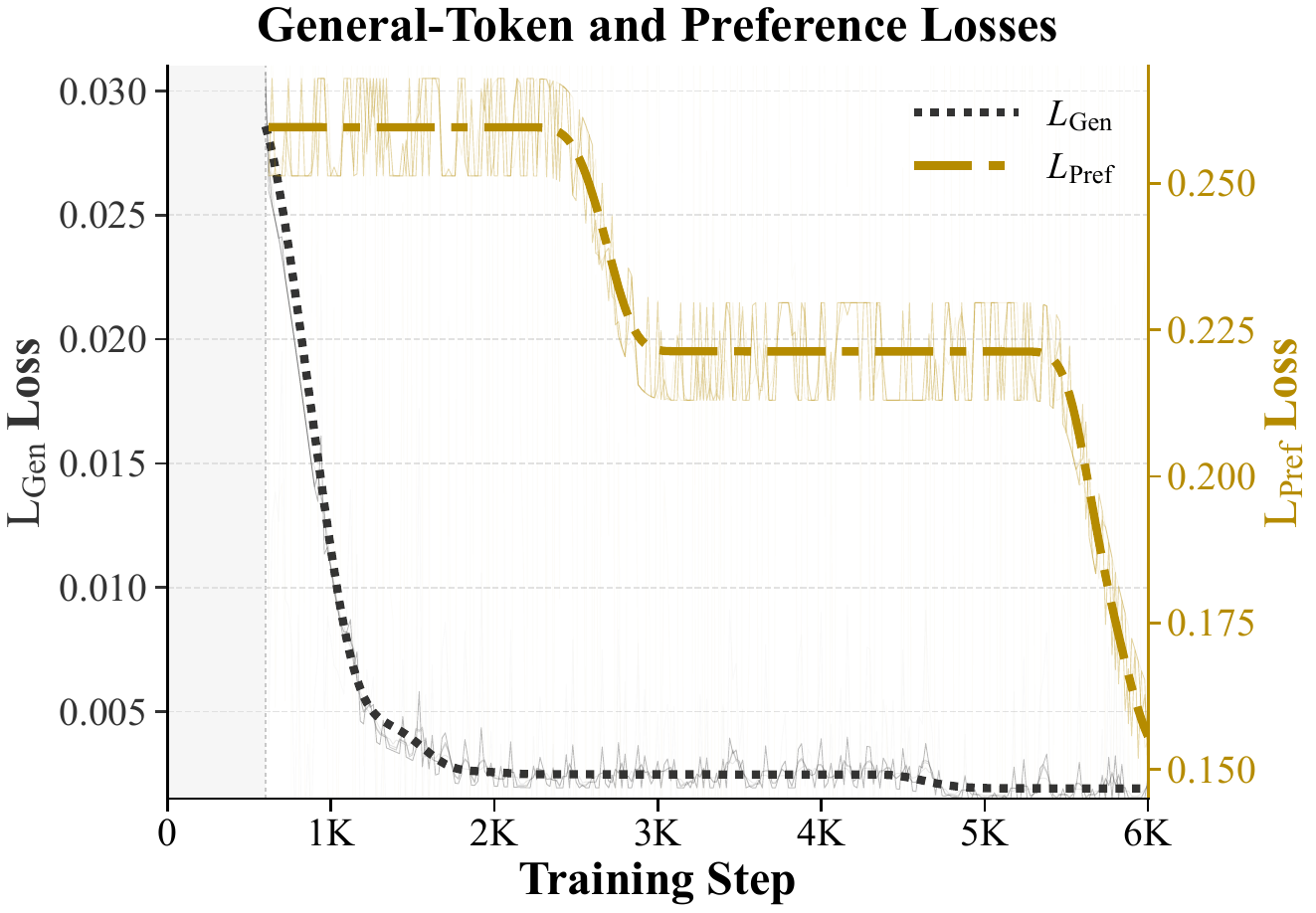}{0.225\textwidth}{25mm}
    \caption{Supplementary training dynamics for BDR. For visualization, we show the first 6K optimization steps; this window does not indicate the total training duration. From left to right:
    total and coordinate cross-entropy losses; coordinate regression and
    format losses; historical and final coordinate entropies; and general
    reconstruction and coordinate-wise preference loss.}
    \label{fig:bdr_losses}
\end{figure}

\subsection{Inference Procedure}
\label{app:inference}

At inference, RER first generates the complete localization state. BDR then
performs a deterministic same-position pass on the current state. After each
pass, only the editable coordinate positions are updated; the surrounding
reasoning text and response structure are retained. The response is then
detokenized and tokenized again, coordinate masks are rebuilt, and the updated
state becomes the next-pass input. We use two BDR passes for all standard
natural-image and remote-sensing evaluations, providing a consistent
refinement budget across settings. Unparseable final \structtag{answer} boxes remain in the evaluation
denominator and receive zero IoU.

\subsection{Prompt and Evaluation Protocols}
\label{app:prompt_protocol}

Figure~\ref{fig:prompt_protocols} consolidates the standard-grounding and
controlled-correction interfaces in one prompt plate.  The structured base
condition and RER use the same structured prompt and parsing rule; the former
uses the frozen backbone directly, whereas the latter activates the RER
adapter.  In controlled correction, every method receives the same image,
referring expression $q$, and deterministic input anchor
$\mathbf b^{\mathrm{in}}$.  The displayed line breaks are typographic, and the
placeholders are replaced by the corresponding sample values at inference.

\begin{figure*}[t]
\centering
\begingroup
\setlength{\fboxsep}{5pt}
\setlength{\fboxrule}{0.55pt}
\newcommand{\protocolbar}[1]{%
  \par\vspace{2.5pt}\noindent
  \colorbox{black!13}{\parbox{\dimexpr\linewidth-2\fboxsep\relax}{\bfseries #1}}%
  \par\vspace{1.5pt}}
\newcommand{\promptrole}[1]{\textbf{#1}\;}
\newcommand{\promptkey}[1]{\textcolor{red!72!black}{\textbf{#1}}}
\fbox{%
\begin{minipage}{0.972\textwidth}
\scriptsize
\raggedright\sloppy
\noindent\colorbox{black}{\parbox{\dimexpr\linewidth-2\fboxsep\relax}{%
\color{white}\bfseries\fontsize{8.2}{9.0}\selectfont Prompt interfaces used in CoEvolve evaluations}}

\begin{minipage}[t]{0.493\linewidth}
\protocolbar{(a) Standard grounding}

\promptrole{Native Qwen3.5 interface}
\structtag{image}\par
Locate every object that matches the description ``$q$'' in the image.
Report bbox coordinates in JSON format using \texttt{[x1, y1, x2, y2]}
normalized to 0--1000.

\protocolbar{Structured Qwen3.5 and RER interface}
\promptrole{User prompt}
You are a precise visual grounding assistant. Locate the target object
described in the text. You must reason step by step, progressively narrowing
from a broad scene region to a tight box. This \promptkey{coarse-to-fine}
process forces you to deeply understand the whole image and the object's
context, preventing misidentification and leading to a more accurate final
localization.

\medskip
Follow the structure and style of the example below.

\medskip
\structtag{refine}\par
Step 1: The scene shows a wooden desk with a laptop, a cup, a blue notebook,
and a blue magazine. $\rightarrow$ [100, 50, 900, 850]\par
Step 2: Focus on the right side of the desk, ignoring the laptop and the cup
on the left; both the blue notebook and the blue magazine are visible in this
area. $\rightarrow$ [400, 200, 850, 750]\par
Step 3: Among the two blue items, the notebook is lying flat near the desk
edge, while the magazine is further back and partially underneath it; identify
the notebook as the closer one. $\rightarrow$ [455, 325, 820, 700]\par
Step 4: Zoom in on the blue notebook itself, excluding the magazine completely;
its cover is facing up with a pen resting beside it. $\rightarrow$
[510, 345, 750, 650]\par
Step 5: Tightly enclose only the blue notebook---its full cover and
spine---trimming away the pen and all surrounding background. $\rightarrow$
[523, 382, 719, 620]\par
\structtag{/refine}\par
\structtag{answer}[523, 382, 719, 620]\structtag{/answer}

\medskip
\promptrole{Requirements}
Use \promptkey{3--5 steps}, each with \promptkey{exactly one bounding box}
\texttt{[x1, y1, x2, y2]} (integers 0--1000). Each box must be tighter than
the previous, covering less background and fitting the target more precisely.
For each step, first provide a brief scene observation locating the target;
then use $\rightarrow$ as a delimiter followed by the box. The final
\structtag{answer} box must be \promptkey{identical to the last step's box}.
Output only the two XML tags; no text outside them. If no matching object
exists, the last step's box and \structtag{answer} must both be \texttt{[]}.

\medskip
\structtag{image}\par
Text description: $q$
\end{minipage}
\hfill
\begin{minipage}[t]{0.493\linewidth}
\protocolbar{(b) BDR refinement}

\promptrole{User prompt}
You are a precise visual grounding refinement assistant. Your job is to
verify and refine bounding box coordinates through multi-step reasoning by
carefully examining the image.

\medskip
\promptrole{Key rules}
\textbullet\ If the bounding box in \structtag{answer} already tightly and
accurately encloses the target object $\rightarrow$
\promptkey{KEEP IT EXACTLY AS-IS}. Do not change coordinates that are already
correct.\par
\textbullet\ The final \structtag{answer} is the authoritative result. The
last step's bbox in \structtag{refine} \promptkey{MUST be consistent} with the
\structtag{answer} bbox.\par
\textbullet\ If you believe the reasoning chain contains errors, you may modify
the relevant steps---update their bounding boxes as needed. You do NOT need to
make every single step's bbox perfectly precise; intermediate steps serve as
progressive reasoning toward the final answer.

\medskip
\promptrole{Guidelines}
\textbullet\ Carefully examine the image to locate the target object described
in the text description below.\par
\textbullet\ If the answer box is already precise $\rightarrow$ output the
SAME reasoning chain unchanged.\par
\textbullet\ If the reasoning chain has inaccuracies $\rightarrow$ refine the
relevant steps and the answer progressively, with each corrected step moving
closer to the correct location.\par
\textbullet\ Keep the same number of steps and overall structure; adjust
coordinates only where they are wrong.\par
\textbullet\ All coordinates use xyxy format normalized to 0--1000,
\promptkey{zero-padded to 4 digits}. For example: 585 $\rightarrow$ ``0585'',
0 $\rightarrow$ ``0000'', 1000 $\rightarrow$ ``1000''; a complete bbox is
\texttt{[0585, 0173, 0797, 0710]}.\par
\textbullet\ The final \structtag{answer} box must tightly enclose only the
target object, trimming away all surrounding background.

\medskip
\structtag{image}\par
Text description: $q$

\protocolbar{(c) Controlled box correction baselines}

\promptrole{Qwen3.5-9B}
A candidate bounding box $\mathbf b^{\mathrm{in}}$ is provided for ``$q$''
but \promptkey{may be inaccurate}. Locate every object that matches the
description ``$q$'' in the image. Report bbox coordinates in JSON format using
\texttt{[x1, y1, x2, y2]} normalized to 0--1000.

\medskip
\promptrole{InternVL3.5-8B}
\structtag{image} Please provide the bounding box coordinate of
\structtag{ref}$q$\structtag{/ref} in this image. A candidate box
$[\mathbf b^{\mathrm{in}}]$ is provided as reference but
\promptkey{may be inaccurate}.

\medskip
\promptrole{CogVLM-Grounding-17B}
The candidate bounding box for $q$ is $[\mathbf b^{\mathrm{in}}]$, but it may
be inaccurate. Can you \promptkey{point out} $q$ in the image and provide the
\promptkey{corrected bounding box} of its location?
\end{minipage}
\end{minipage}}
\endgroup
\caption{Exact prompt templates used for standard grounding, BDR refinement,
and controlled box correction. Red bold text marks the constraints that define
the output interface or the preserve-or-correct behavior. Only line wrapping
and sample-specific placeholders are changed for typesetting.}
\label{fig:prompt_protocols}
\end{figure*}

\paragraph{Reproduced standard-grounding baselines.}
The reproduced entries in Table~\ref{tab:stringent_iou_v4} use each model's
native grounding interface rather than the structured RER prompt.  Their
shared evaluation conventions are summarized in
Table~\ref{tab:app_reproduced_protocols}.  Each model produces one
autoregressive response per example.  Predictions are evaluated as
normalized $xyxy$ coordinates in $[0,1000]$; a single complete final box is
extracted, while invalid or unparseable outputs remain in the denominator
with zero IoU.

\begin{table}[!ht]
\centering
\setlength{\belowcaptionskip}{5pt}
\caption{Protocols for the locally reproduced standard-grounding baselines
in Table~\ref{tab:stringent_iou_v4}.  ``Native'' denotes the prompt interface
distributed or recommended for the corresponding model family.}
\label{tab:app_reproduced_protocols}
\small
\setlength{\tabcolsep}{3.5pt}
\begin{tabular*}{\columnwidth}{@{\extracolsep{\fill}}ll@{}}
\toprule
Model & Native prompt interface \\
\midrule
InternVL3.5-8B & Referring box with \texttt{<ref>} expression \\
Qwen3-VL-8B & JSON grounding with normalized box \\
Qwen3.5-9B & JSON grounding with normalized box \\
CogVLM-Grounding-17B & Caption-to-box grounding \\
\bottomrule
\end{tabular*}
\end{table}

\paragraph{Native-output-derived state construction.}
For the Native-Output BDR condition, unadapted Qwen3.5-9B first produces
visible natural-language analysis and a terminal box through the native grounding
interface in Figure~\ref{fig:prompt_protocols}. Auxiliary \structtag{think}
content is removed. The remaining visible analysis, including its terminal box,
is retained verbatim inside a \structtag{refine} block; the same zero-padded box
is placed in \structtag{answer}. For example, the processed state for the query
``the man in yellow coat'' is
\begin{quote}
\small
\textbf{Example native-output-derived state.}
\structtag{refine} Got it, let's look at the image. The man in the yellow coat
is the one on the right. So we need to find his bounding box. Let's check the
coordinates. The yellow coat, red pants, and the bib. So the region is
\texttt{[0585, 0185, 0793, 0696]}. \structtag{/refine}
\structtag{answer}\texttt{[0585, 0185, 0793, 0696]}\structtag{/answer}
\end{quote}
This state contains a terminal localization hypothesis but
no RER-trained ordered region commitments. BDR keeps the \structtag{refine}
content fixed and performs same-position editing only on the 16 Answer digits.
A separate native-output editor is trained from scratch on these
native-output-derived training states; its parameters are independent of the
BDR trained on RER-generated states.
Training uses Answer coordinate cross-entropy, soft-coordinate regression,
correctness-gated Answer entropy, coordinate-wise preference, format
cross-entropy, and general reconstruction of the unchanged \structtag{refine}
content. The annotated box is used only for target and loss construction and is
never inserted into the model input.

\paragraph{Controlled-correction BDR state.}
For BDR, the shared corrupted anchor is represented exactly as
\begin{center}
\small\ttfamily
\structtag{refine} The target object is located at coordinates
$[\mathbf b^{\mathrm{in}}]$. \structtag{/refine}
\structtag{answer}$[\mathbf b^{\mathrm{in}}]$\structtag{/answer}
\end{center}
There are no earlier numbered Steps. The two occurrences of
$\mathbf b^{\mathrm{in}}$ form the editable final-refinement and Answer roles,
while the sentence, tags, and punctuation remain fixed. Autoregressive
baselines receive the same image, expression, and corrupted anchor through the
templates in Figure~\ref{fig:prompt_protocols}. No method receives the annotated
box as input.

For an image of width $W$ and height $H$, pixel coordinates are normalized as
\begin{equation}
x'=\operatorname{round}(1000x/W),\qquad
y'=\operatorname{round}(1000y/H),
\label{eq:app_coordinate_normalization}
\end{equation}
and clipped to $[0,1000]$.  RER and the native baseline use deterministic
decoding for reported evaluation, while BDR applies digit-vocabulary argmax at
the aligned coordinate positions.  The evaluator parses one complete legal
box from the final \structtag{answer} field.  Invalid or unparseable predictions
remain in the corresponding split denominator and receive zero IoU.

\subsection{Implementation Notes}
\label{app:hyperparameters}

We implement RER and BDR on Qwen3.5-9B and train their adapters sequentially.
RER and BDR initialize separate adapters, and standard evaluation uses one RER
generation followed by two deterministic BDR passes.  Complete optimization
configurations, adapter settings, execution commands, and hardware details are
provided in the anonymous supplementary code.  All method-defining objectives,
target rules, perturbation procedures, decoding protocols, and evaluation
conventions are specified in the paper and appendices.

\section{Formal Properties}
\label{app:formal_properties}

The following propositions characterize boundedness, normalization, gradient
direction, and entropy gating in the implemented objectives.

\begin{proposition}[Bounded progressive signal]
The implemented progressive component satisfies
$0\leq C_{\mathrm{prog}}\leq1$.  Consequently,
$0\leq R_{\mathrm{prog}}\leq R_{\mathrm{fmt}}$.
\end{proposition}
\begin{proof}
Let $m_i^{u}=\max_{j\leq i}u_j$ and $m_i^{q}=\max_{j\leq i}q_j$.
Every positive best-so-far increment telescopes, and therefore
\begin{equation}
\sum_{i=2}^{K}\frac{m_i^{u}-m_{i-1}^{u}}{1-u_1+\epsilon}\leq1,
\qquad
\sum_{i=2}^{K}\frac{m_i^{q}-m_{i-1}^{q}}{1-q_1+\epsilon}\leq1.
\end{equation}
The non-positive increments retained in Eq.~(\ref{eq:app_progress}) can only
decrease the sum.  Its $0.6/0.4$ convex weighting and outer truncation thus
give $C_{\mathrm{prog}}\in[0,1]$.  The coverage, overall-gain, and consistency
terms also lie in $[0,1]$, and their coefficients in
Eq.~(\ref{eq:app_prog_reward}) sum to one.  Multiplication by
$R_{\mathrm{fmt}}\in[0,1]$ proves the second bound.
\end{proof}

\begin{proposition}[Normalized coordinate weighting]
For positive coordinate weights, the weighted box cross-entropy in
Eq.~(\ref{eq:app_weighted_box_ce}) is invariant to a common positive scaling
of all weights and remains between the minimum and maximum selected token
losses.
\end{proposition}
\begin{proof}
Writing $c_t$ for per-token cross-entropy and
$\alpha_t=w_t/\sum_jw_j$ gives
$\mathcal L_{\mathrm{box}}=\sum_t\alpha_tc_t$, where
$\alpha_t>0$ and $\sum_t\alpha_t=1$.  The loss is therefore a convex
combination of the selected token losses.  Replacing every $w_t$ by
$\lambda w_t$ for $\lambda>0$ leaves each $\alpha_t$ unchanged, which proves
both claims.
\end{proof}

\begin{proposition}[Preference-gradient direction]
For every active coordinate, gradient descent increases the output-versus-input
margin for a beneficial edit and decreases it for a detrimental edit.
\end{proposition}
\begin{proof}
Differentiating Eq.~(\ref{eq:app_pref_bce}) yields
\begin{equation}
\frac{\partial\ell_{sk}}{\partial D_{sk}}=\sigma(D_{sk})-y_{sk}.
\end{equation}
If $y_{sk}=1$, this derivative is negative and gradient descent increases
$D_{sk}$.  If $y_{sk}=0$, it is positive and gradient descent decreases
$D_{sk}$.  Coordinates inside the dead zone
$|\Delta_{sk}|\leq0.005$ are inactive and receive no preference gradient.
\end{proof}

\begin{proposition}[Correctness-gated entropy]
The final-coordinate entropy term cannot directly sharpen a digit whose
digit-vocabulary argmax differs from the selected target digit.
\end{proposition}
\begin{proof}
The gate includes a final digit position only when its current argmax equals
the selected target.  An incorrect position is therefore excluded from
$\mathcal H_{\mathrm{fin}}$ and has zero gradient from this term.  Included
positions minimize entropy only after satisfying the discrete correctness
condition.  The historical entropy term is ungated but contains no absolute
coordinate target.
\end{proof}

\paragraph{Training and inference separation.}
The annotation $\mathbf b^{*}$ is used in training to select targets and
compute edit labels.  At inference, target selection, entropy gating, and
preference labels are disabled; the model observes only $(I,q)$ and the
current response.

\section{Additional Experimental Results}
\label{app:additional_results}

\subsection{Dataset Accounting}
\label{app:dataset_accounting}

The archived localization evaluation contains the following eight standard
splits.  The high-precision, controlled-correction, and ablation studies use
the five test splits shown in the final column.
\begin{table}[!htbp]
\centering
\setlength{\belowcaptionskip}{5pt}
\caption{Dataset composition used by the archived evaluations.}
\label{tab:app_dataset}
\small
\begin{tabular}{l l r c c}
\toprule
Dataset & Split & Samples & Eight-split evaluation & Five-test evaluation\\
\midrule
RefCOCO & val & 10,834 & \checkmark & --\\
RefCOCO & testA & 5,657 & \checkmark & \checkmark\\
RefCOCO & testB & 5,095 & \checkmark & \checkmark\\
RefCOCO+ & val & 10,758 & \checkmark & --\\
RefCOCO+ & testA & 5,726 & \checkmark & \checkmark\\
RefCOCO+ & testB & 4,889 & \checkmark & \checkmark\\
RefCOCOg & val & 4,896 & \checkmark & --\\
RefCOCOg & test & 9,602 & \checkmark & \checkmark\\
\midrule
\textbf{Total} & eight / five splits & \textbf{57,457 / 30,969}
& &\\
\bottomrule
\end{tabular}
\end{table}

For the main experiment, \textit{Overall} is the arithmetic mean of the
eight split-level $\mathrm{Acc}@0.5$ values.  The RER aggregate in
Table~\ref{tab:app_stage1} is computed jointly over all 57,457 localization
examples.  Stricter-IoU, controlled-correction, meaningful-state ablation,
and state-source aggregates are computed jointly over the 30,969 examples
from the five test splits.  The segmentation experiment uses 57,457
expression instances and retains this denominator separately from the
localization tables.

\subsection{RER Localization State Quality}
\label{app:stage1_results}

Table~\ref{tab:app_stage1} reports the terminal answer extracted from the
RER state before BDR refinement.  Unparseable predictions are assigned zero
IoU and remain in the split denominators.
\begin{table}[!htbp]
\centering
\setlength{\belowcaptionskip}{5pt}
\caption{RER terminal localization on the eight standard splits. The
aggregate row is computed jointly across the splits.}
\label{tab:app_stage1}
\small
\begin{tabular}{llrrrrr}
\toprule
Dataset & Split & Samples & mIoU & Acc@0.5 & Acc@0.7 & Acc@0.9\\
\midrule
RefCOCO & val & 10,834 & 83.35 & 90.10 & 83.81 & 63.61\\
RefCOCO & testA & 5,657 & 84.09 & 90.75 & 85.56 & 64.27\\
RefCOCO & testB & 5,095 & 81.37 & 88.38 & 80.73 & 59.02\\
RefCOCO+ & val & 10,758 & 77.14 & 82.57 & 76.15 & 57.82\\
RefCOCO+ & testA & 5,726 & 79.32 & 85.23 & 79.03 & 59.10\\
RefCOCO+ & testB & 4,889 & 74.18 & 80.18 & 72.20 & 53.12\\
RefCOCOg & val & 4,896 & 82.49 & 89.22 & 82.15 & 60.72\\
RefCOCOg & test & 9,602 & 82.48 & 89.28 & 82.70 & 61.45\\
\midrule
\textbf{Aggregate} & \textbf{all splits} & \textbf{57,457}
& \textbf{80.68} & \textbf{87.06} & \textbf{80.48} & \textbf{60.23}\\
\bottomrule
\end{tabular}
\end{table}

Terminal answer quality alone does not determine the final performance of a
source-matched state--editor pipeline. Table~\ref{tab:app_stage1_backbone} contrasts the native grounding
output of Qwen3.5-9B, the same model under the structured RER interface without
RER adaptation, the trained RER state, and refinement from the two state
sources. The structured interface alone is substantially weaker than the
native grounding interface, whereas RER adapts the policy to the explicit
trajectory format and recovers most of the terminal localization quality.
Each state source is paired with a separately trained BDR. The native-output
pipeline starts at 82.16\% terminal mIoU, whereas the RER pipeline starts at
80.70\%. After source-matched refinement, RER + BDR gains 5.02 points to
reach 85.72\% mIoU and attains 68.64\% $\mathrm{Acc}@0.9$. The native-output
pipeline gains 2.41 points to reach 84.57\% mIoU and attains 66.32\%
$\mathrm{Acc}@0.9$. This comparison highlights the effectiveness of the RER$\rightarrow$BDR
pairing: despite starting from a weaker terminal box, it obtains larger
refinement gains and higher final localization quality. The zero-anchor
ablation in Fig.~\ref{fig:zero_ablation} provides complementary evidence that a
meaningful state and its anchor-aware objectives matter to the complete system.

\begin{table*}[ht]
\centering
\setlength{\belowcaptionskip}{5pt}
\caption{Terminal grounding and source-matched refinement under the
five-test metric aggregation. Metrics are computed jointly across the splits;
invalid or unparseable outputs receive zero IoU, and BDR rows use the
second refinement pass. Each refined state source uses a separately trained BDR,
and the native and structured rows use their corresponding prompts in
Section~\ref{app:prompt_protocol}.}
\label{tab:app_stage1_backbone}
\scriptsize
\setlength{\tabcolsep}{4.5pt}
\begin{tabular}{lllrrr}
\toprule
State source & Interface & RER adaptation & mIoU & $\mathrm{Acc}@0.7$ & $\mathrm{Acc}@0.9$\\
\midrule
Qwen3.5-9B & native output & No & 82.16 & 83.31 & 60.64\\
Structured Qwen3.5-9B & structured & No & 63.48 & 57.08 & 24.28\\
RER terminal answer (ours) & structured & Yes & 80.70 & 80.56 & 59.81\\
Native-Output + BDR (ours) & native-output-derived & No & \underline{84.57} & \underline{86.56} & \underline{66.32}\\
\textbf{RER state + BDR (ours)} & structured & Yes & \textbf{85.72} & \textbf{87.94} & \textbf{68.64}\\
\bottomrule
\end{tabular}
\end{table*}

Table~\ref{tab:app_ablation_split} gives the split-level results of the
meaningful-state ablation.
\begin{table*}[ht]
\centering
\setlength{\belowcaptionskip}{5pt}
\caption{Split-level meaningful-state ablation. The upper block reports mIoU / $\mathrm{Acc}@0.5$ and the lower block reports $\mathrm{Acc}@0.7$ / $\mathrm{Acc}@0.9$.}
\label{tab:app_ablation_split}
\scriptsize
\setlength{\tabcolsep}{3.0pt}
\begin{tabular*}{\textwidth}{@{\extracolsep{\fill}}lccccc@{}}
\toprule
Method & RefCOCO A & RefCOCO B & RefCOCO+ A & RefCOCO+ B & RefCOCOg test \\
\midrule
\textbf{CoEvolve, mIoU / Acc@0.5}
& \textbf{88.85 / 95.28} & \textbf{85.41 / 92.38}
& \textbf{87.12 / 93.45} & \textbf{81.23 / 87.79}
& \textbf{85.47 / 92.75} \\
Zero-Anchor BDR, mIoU / Acc@0.5
& \underline{87.91 / 94.22} & \underline{84.24 / 90.62}
& \underline{86.00 / 92.19} & \underline{78.89 / 84.27}
& \underline{81.34 / 89.79} \\
\midrule
\textbf{CoEvolve, Acc@0.7 / Acc@0.9}
& \textbf{92.65 / 74.24} & \textbf{87.05 / 67.79}
& \textbf{90.67 / 73.11} & \textbf{82.33 / 64.68}
& \textbf{86.86 / 65.15} \\
Zero-Anchor BDR, Acc@0.7 / Acc@0.9
& \underline{90.26 / 69.91} & \underline{84.02 / 65.32}
& \underline{87.91 / 67.69} & \underline{77.79 / 60.20}
& \underline{81.97 / 53.62} \\
\bottomrule
\end{tabular*}
\end{table*}

\FloatBarrier

\subsection{Qualitative Grounding Results}
\label{app:qualitative_results}

Figure~\ref{fig:app_structured_edit} exposes one complete state transition.
The RER response follows the required \structtag{refine} and
\structtag{answer} format, and every reasoning Step terminates in an explicit
region.  BDR then edits coordinates at the same response positions: the
coordinate matrix shows that refinement can update historical Steps, the last
Step, and the Answer while retaining the surrounding reasoning text.
\begin{figure}[H]
    \centering
    \includegraphics[width=\linewidth]{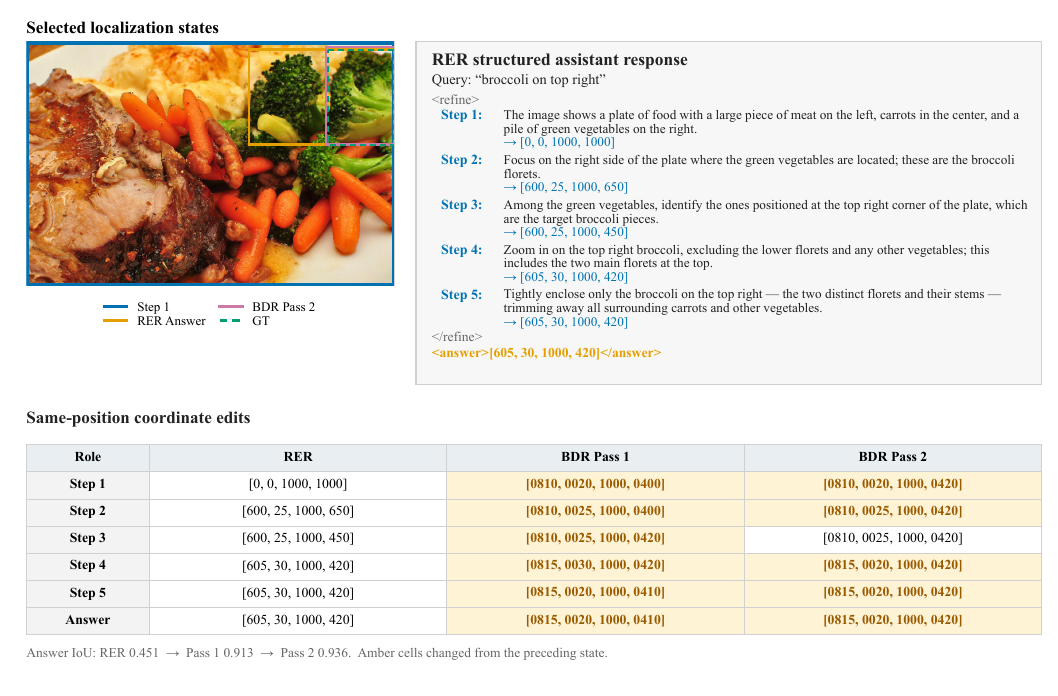}
    \caption{A structured RER trajectory and its subsequent BDR coordinate edits
    for the query ``broccoli on top right.'' RER constructs a progressive
    localization trajectory, while BDR performs same-position coordinate
    refinement across multiple editable roles and substantially improves the
    final Answer. The image shows the first RER region, RER Answer, second-pass
    BDR Answer, and annotation. The lower matrix traces all coordinate states;
    amber cells mark edits from the preceding state.}
    \label{fig:app_structured_edit}
\end{figure}

\clearpage
Figure~\ref{fig:app_qualitative_natural} visualizes additional RER-to-BDR
localization updates on natural images.  The cases are chosen to expose
distinct correction behaviors rather than to estimate their frequency.  Each
panel preserves the full image, and its IoU is computed from the archived box
shown in that panel.

\begin{figure}[H]
    \centering
    \includegraphics[width=\linewidth]{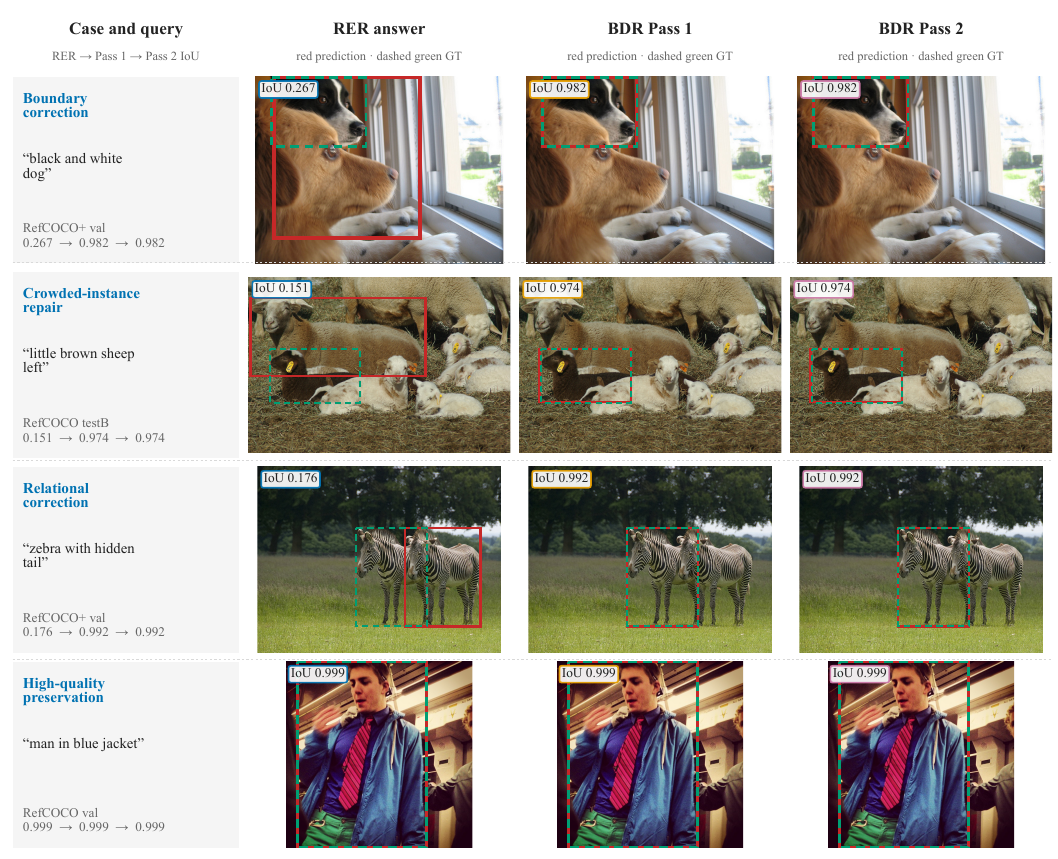}
    \caption{Qualitative results on the RefCOCO family. Red solid boxes
    show model predictions, green dashed boxes show annotations, and badges
    report per-example IoU. The examples illustrate boundary correction,
    disambiguation among nearby instances, relational
    correction, and preservation of an accurate RER anchor.}
    \label{fig:app_qualitative_natural}
\end{figure}

\clearpage
\subsection{Controlled Correction Details}
\label{app:strict_and_controlled}

The controlled inputs are one deterministic nonzero perturbation per example,
with target-IoU centers $0.20,0.40,0.60,0.80,0.94$ and four transformation
types (translation, scale, aspect-ratio, and mixed).  The same anchor is
given to every model.  CoEvolve uses one BDR pass in this experiment.  Its
$\mathrm{Repair}@0.7$ is $87.99\%$, defined on the subset whose input IoU is below
$0.7$.
Using the fixed seed 20260803, the construction procedure deterministically
assigns each example a target-IoU band, perturbation type, and transformation
parameters.

Figure~\ref{fig:app_controlled_construction} summarizes the benchmark
construction.  The five test splits contain 30,969 examples and have closely
matched input mIoU values, yielding an aggregate input mIoU of $58.80$ over
all examples.
The deterministic assignment is approximately uniform across the five target
IoU bands and the four perturbation types, preventing the aggregate result
from being dominated by one corruption severity or geometric transformation.

\begin{figure}[!htbp]
    \centering
    \includegraphics[width=0.82\textwidth]{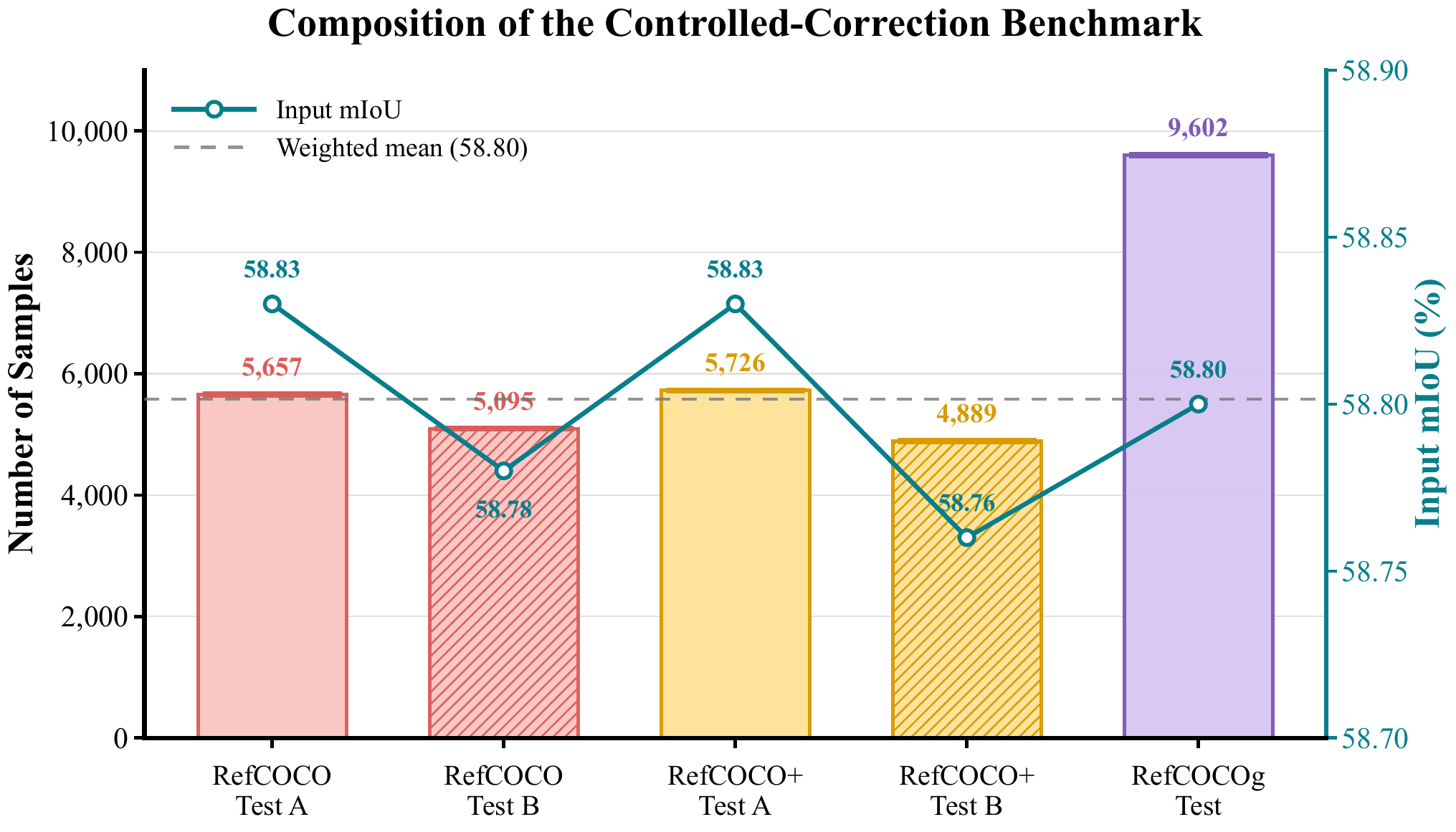}\par\medskip
    \includegraphics[width=0.82\textwidth]{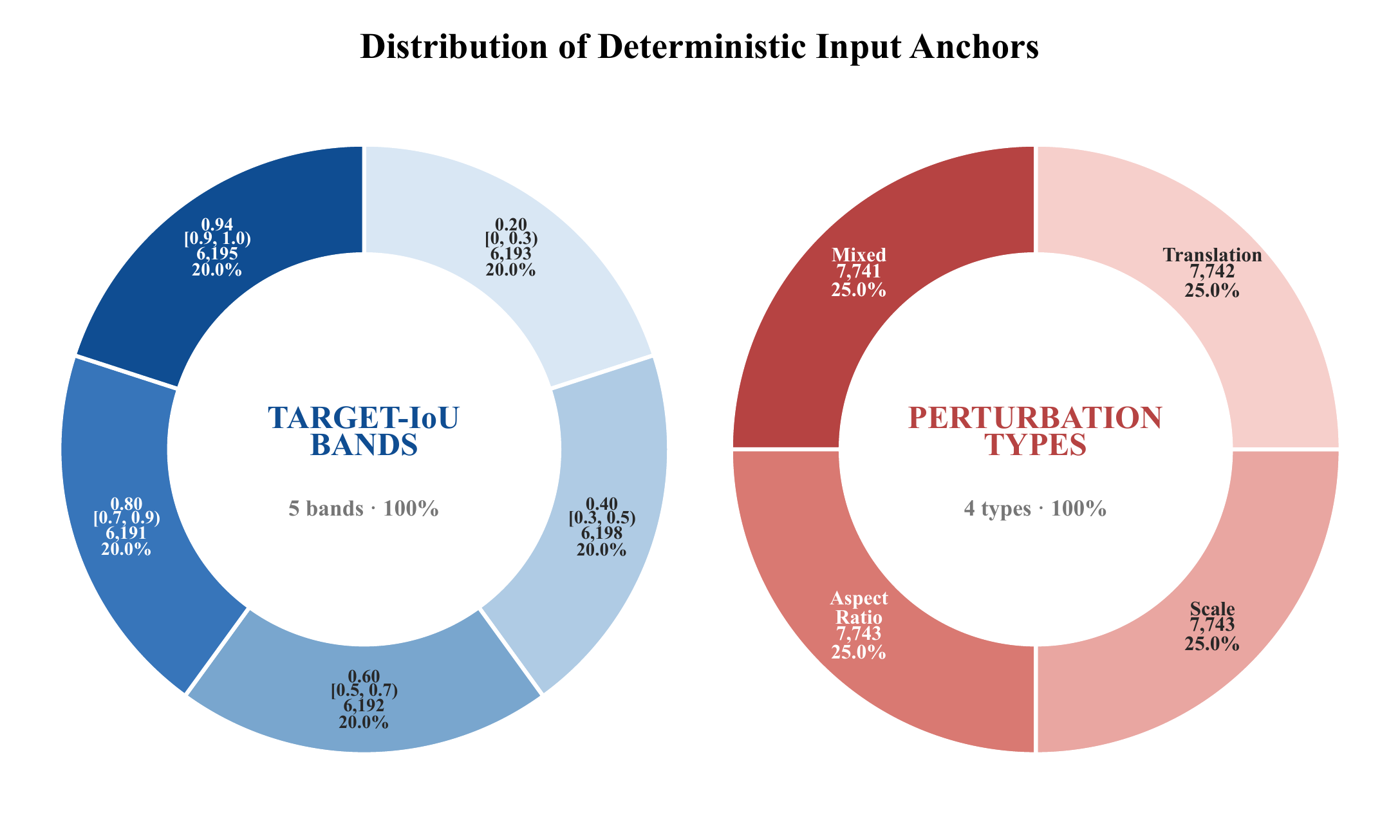}
    \caption{Construction of the controlled-correction benchmark.  The upper
    panel reports the number of examples and input mIoU for each test split;
    the dashed line marks the aggregate input mIoU over all examples.  The lower panel
    reports the assignment across target IoU bands and perturbation types.
    Each example contributes one deterministic nonzero anchor.}
    \label{fig:app_controlled_construction}
\end{figure}

\FloatBarrier

\subsection{Remote-Sensing Grounding}
\label{app:remote}

DIOR-RSVG evaluation contains 7,500 expression instances.  The main paper
compares CoEvolve with prior methods, while Table~\ref{tab:app_remote} traces
the three saved localization states. Unparseable outputs remain in the
denominator and receive zero IoU.
RER and BDR retain the natural-image reward, loss, and two-pass editing
definitions while being retrained on the DIOR-RSVG training split. Their
domain-specific optimization configurations are provided in the anonymous
supplementary code.
The reproduced native Qwen3.5-9B row in the main comparison is evaluated
without DIOR-RSVG adaptation and is included as a base-to-system reference,
not as a matched-training architectural comparison.
\begin{table}[!htbp]
\centering
\setlength{\belowcaptionskip}{5pt}
\caption{CoEvolve localization-state progression on DIOR-RSVG test expressions. RER and BDR are retrained on the DIOR-RSVG training split.}
\label{tab:app_remote}
\scriptsize
\setlength{\tabcolsep}{3.0pt}
\begin{tabular}{lrrrr}
\toprule
State & mIoU & $\mathrm{Pr}@0.5$ & $\mathrm{Pr}@0.7$
& $\mathrm{Pr}@0.9$\\
\midrule
RER output & 58.2248 & 64.2400 & 51.1200 & 23.0400\\
BDR Pass 1 & 75.4161 & 85.2533 & 74.6267 & 41.5200\\
\textbf{BDR Pass 2} & \textbf{76.2295} & \textbf{86.2000}
& \textbf{76.1600} & \textbf{43.4533}\\
\bottomrule
\end{tabular}
\end{table}

Figure~\ref{fig:app_qualitative_remote} complements the aggregate progression
with successful correction modes and a non-monotonic failure case.
\begin{figure}[H]
    \centering
    \includegraphics[width=0.82\linewidth]{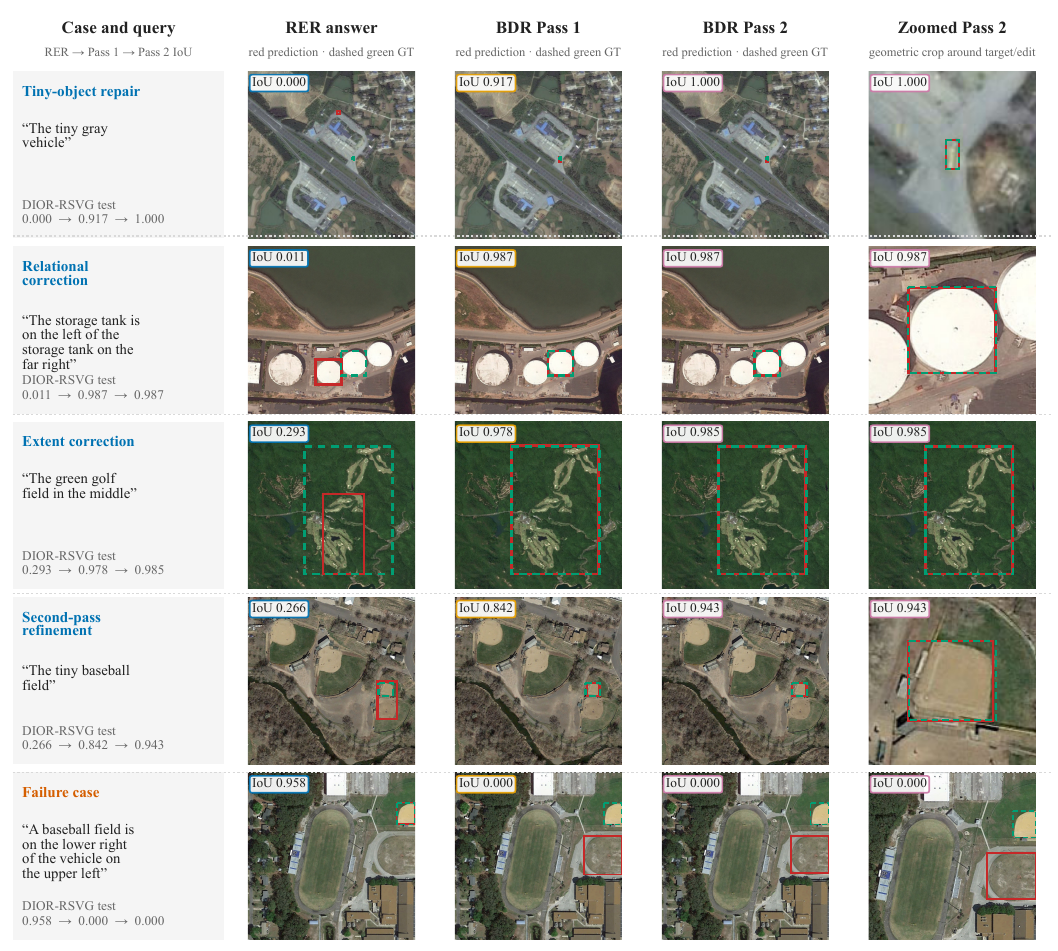}
    \caption{Qualitative results on DIOR-RSVG. Red solid boxes show model
    predictions and green dashed boxes show annotations. The first four rows show repair
    of a tiny object, a relationally specified instance, object extent, and a
    boundary further improved by the second BDR pass.  The final row shows a
    failure in which refinement moves away from an initially accurate anchor.
    The rightmost column is a geometric crop around the annotation and final
    edit; no contrast enhancement is applied.}
    \label{fig:app_qualitative_remote}
\end{figure}

\FloatBarrier

\subsection{Downstream Segmentation}
\label{app:segmentation}

We feed each final localization box to a frozen SAM3 without mask fine-tuning.
The segmentation archive contains 57,457 expression instances.
Table~\ref{tab:app_segmentation_comparison} reports split-level cIoU and the
arithmetic mean across the eight splits.

\begin{table}[!htbp]
\centering
\setlength{\belowcaptionskip}{5pt}
\caption{Referring-expression segmentation cIoU on the eight standard
splits.  Public methods use their reported training and pretraining
protocols; Macro-8 is the arithmetic mean of the eight split values.
Best and second-best results are bold and underlined.}
\label{tab:app_segmentation_comparison}
\scriptsize
\setlength{\tabcolsep}{2.0pt}
\begin{tabular*}{\textwidth}{@{\extracolsep{\fill}}lrrrrrrrrr@{}}
\toprule
& \multicolumn{3}{c}{RefCOCO} & \multicolumn{3}{c}{RefCOCO+}
& \multicolumn{2}{c}{RefCOCOg} & \\
\cmidrule(lr){2-4}\cmidrule(lr){5-7}\cmidrule(lr){8-9}
Method & val & A & B & val & A & B & val & test & Macro-8\\
\midrule
CRIS~\citep{wang2022cris} & 70.47 & 73.18 & 66.10 & 62.27 & 68.08
& 53.68 & 59.87 & 60.36 & 64.25\\
LAVT~\citep{yang2022lavt} & 72.73 & 75.82 & 68.79 & 62.14 & 68.38
& 55.10 & 61.24 & 62.09 & 65.79\\
CGFormer~\citep{tang2023cgformer} & 74.75 & 77.30 & 70.64 & 64.54
& 71.00 & 57.14 & 64.68 & 65.09 & 68.14\\
ReLA~\citep{liu2023gres} & 73.82 & 76.48 & 70.18 & 66.04 & 71.02
& 57.65 & 65.00 & 65.97 & 68.27\\
MagNet~\citep{chng2024magnet} & 75.24 & 78.24 & 71.05 & 66.16 & 71.32
& 58.14 & 65.36 & 66.03 & 68.94\\
LISA~\citep{lai2024lisa} & 74.90 & 79.10 & 72.30 & 65.10 & 70.80
& 58.10 & 67.90 & 70.60 & 69.85\\
PolyFormer-B~\citep{liu2023polyformer} & 74.82 & 76.64 & 71.06 & 67.64
& 72.89 & 59.33 & 67.76 & 69.05 & 69.90\\
OneRef-L~\citep{xiao2024oneref} & 80.48 & 82.78 & 78.27 & 74.25
& 78.41 & 69.85 & 74.91 & 77.36 & 77.04\\
EVF-SAM~\citep{zhang2024evfsam} & 82.10 & 83.70 & 80.00 & 75.20
& 78.30 & 70.10 & 76.80 & 77.40 & 77.95\\
\textbf{CoEvolve + SAM3 (ours)} & 80.28 & 82.81 & 77.87 & 76.22
& 80.60 & 72.42 & 77.06 & 78.23 & 78.19\\
LENS~\citep{zhu2026lens} & 84.20 & 85.30 & 81.00 & 79.40
& 82.80 & 74.30 & 81.20 & 81.00 & 81.15\\
HyperSeg~\citep{wei2025hyperseg} & 84.80 & 85.70 & 83.40 & 79.00
& 83.50 & 75.20 & 79.40 & 78.90 & 81.24\\
InstructSeg~\citep{wei2025instructseg} & \textbf{85.80} & 86.60
& \underline{84.00} & \underline{80.10} & \underline{83.80}
& \underline{75.60} & 79.30 & 80.30 & 81.94\\
X-SAM~\citep{wang2026xsam} & \underline{85.10} & \underline{87.10}
& 83.40 & 78.00 & 81.00 & 74.40 & \underline{83.80}
& \underline{83.90} & \underline{82.09}\\
FlowSeg~\citep{zhang2026flowseg} & \textbf{85.80} & \textbf{87.40}
& \textbf{84.80} & \textbf{80.20} & \textbf{84.50}
& \textbf{76.90} & \textbf{86.50} & \textbf{86.10} & \textbf{84.03}\\
\bottomrule
\end{tabular*}
\end{table}

\section{Limitations}
\label{app:limitations}

\paragraph{Imperfect preserve-or-edit decisions.}
BDR improves correction behavior on average but does not guarantee an
improvement for every example.  In controlled correction, refinement degrades
$15.50\%$ of the anchors, compared with $11.42\%$ for Qwen3.5-9B.
Quality-gated targets encourage preservation of accurate inputs, but the
learned preserve-or-edit decision remains imperfect for a residual subset.
Calibrated edit acceptance, uncertainty-aware refinement, or adaptive stopping
could retain the current state when a further update is unlikely to help.

\paragraph{State--editor pairing and broader transfer.}
The source-matched comparison shows that RER$\rightarrow$BDR achieves stronger
final refinement than the separately trained native-output pipeline. This result
motivates further study of how state representations interact with their
corresponding editors, particularly across upstream model families and response
formats.

\paragraph{Coordinate-centric editing.}
BDR treats reasoning text as fixed semantic context and edits the coordinate
fields of the localization state. This design keeps the semantic context fixed during position-aligned refinement, but semantic errors that cannot be resolved through coordinate correction remain outside the current editor’s scope. The current objective also does not
explicitly enforce semantic-role consistency between an earlier reasoning step
and its edited historical box. Extending bidirectional editing
toward jointly editable semantic--spatial states is a promising direction.

\end{document}